\def\isarxiv{1} 

\ifdefined\isarxiv
\documentclass[11pt]{article}

\else

\documentclass{article}

\usepackage{neurips_2023}

\fi

\pdfoutput=1
\usepackage{enumitem}
\usepackage{textgreek}
\usepackage{units}
\usepackage{amsmath}
\usepackage{amsthm}
\usepackage{amssymb}
\usepackage{algorithm}
\usepackage{subfig}
\usepackage{color}
\usepackage[english]{babel}
\usepackage{graphicx}
\usepackage{grffile}
\usepackage[square,sort,comma,numbers]{natbib}
\usepackage{wrapfig,epsfig}
\usepackage{url}
\usepackage{epstopdf}
\usepackage{algpseudocode}
\usepackage[T1]{fontenc}
\usepackage{bbm}
\usepackage{comment}
\usepackage{dsfont}
\usepackage{bm}

\ifdefined\isarxiv

\usepackage[margin=1in]{geometry}
 \usepackage{hyperref}
\definecolor{mydarkblue}{rgb}{0,0.08,0.45}
\hypersetup{colorlinks=true, citecolor=mydarkblue,linkcolor=mydarkblue}
\else 
 \usepackage{microtype}
 \usepackage{hyperref}
\definecolor{mydarkblue}{rgb}{0,0.08,0.45}
\hypersetup{colorlinks=true, citecolor=mydarkblue,linkcolor=mydarkblue}
\fi

\newtheorem{theorem}{Theorem}[section]
\newtheorem{lemma}[theorem]{Lemma}
\newtheorem{definition}[theorem]{Definition}

\newtheorem{corollary}[theorem]{Corollary}

\newtheorem{assumption}[theorem]{Assumption}

\newtheorem{claim}[theorem]{Claim}

\definecolor{b2}{RGB}{51,153,255}
\definecolor{mygreen}{RGB}{80,180,0}
\definecolor{yl}{RGB}{255,80,0}

\newcommand{\mat}[1]{\bm{#1}}
\DeclareMathOperator*{\E}{{\mathbb{E}}}
\graphicspath{{./figs/}}

\newcommand{\wt}{\widetilde}
\newcommand{\ov}{\overline}
\newcommand{\eps}{\varepsilon}
\renewcommand{\epsilon}{\varepsilon}
\newcommand{\N}{\mathcal{N}}
\newcommand{\R}{\mathbb{R}}

\renewcommand{\k}{\mathsf{K}}

\renewcommand{\bar}{\ov}
\renewcommand{\d}{\mathrm{d}}
\newcommand{\poly}{\mathrm{poly}}

\newcommand{\vect}{\mathrm{vec}}

\newcommand{\nn}{\mathrm{nn}}
\newcommand{\gnn}{\mathrm{gnn}}
\newcommand{\node}{\mathrm{node}}
\newcommand{\train}{\mathrm{train}}

\newcommand{\gntk}{\mathrm{gntk}}
\newcommand{\init}{\mathrm{init}}
\newcommand{\test}{\mathrm{test}}
\newcommand{\tr}{\mathrm{tr}}
\newcommand{\neighbor}{\mathcal{N}}

\newcommand{\aggregate}{\textsc{Aggregate}}
\newcommand{\combine}{\textsc{Combine}}
\newcommand{\readout}{\textsc{ReadOut}}
\newcommand{\G}{\mathsf{G}}
\newcommand{\K}{\mathsf{K}}
\renewcommand{\H}{\mathsf{H}}
\newcommand{\cts}{\mathrm{cts}}
\newcommand{\dis}{\mathrm{dis}}
\DeclareMathOperator{\unif}{unif}

\newcommand{\diag}{\mathrm{diag}}

\begin{document}

\ifdefined\isarxiv
\date{}
    \title{Is Solving Graph Neural Tangent Kernel Equivalent to Training Graph Neural Network?}
    \author{
    Lianke Qin\thanks{\texttt{lianke@ucsb.edu}. UCSB.}
    \and 
    Zhao Song\thanks{\texttt{zsong@adobe.com}. Adobe.}
    \and
    Baocheng Sun\thanks{ \texttt{woafrnraetns@gmail.com}.  Weizmann Institute of Science.}
    }
\else 

\title{Is Solving Graph Neural Tangent Kernel Equivalent to Training Graph Neural Network?}

\fi

\ifdefined\isarxiv
\begin{titlepage}
  \maketitle
  \begin{abstract}
A rising trend in theoretical deep learning is to understand why deep learning works through Neural Tangent Kernel (NTK)~\cite{jgh18}, a kernel method that is equivalent to using gradient descent to train a multi-layer infinitely-wide neural network. NTK is a major step forward in the theoretical deep learning because it allows researchers to use traditional mathematical tools to analyze properties of deep neural networks and to explain various neural network techniques from a theoretical view. A natural extension of NTK on graph learning is \textit{Graph Neural Tangent Kernel (GNTK)}, and researchers have already provide GNTK formulation for graph-level regression and show empirically that this kernel method can achieve similar accuracy as GNNs on various bioinformatics datasets~\cite{dhs+19}. The remaining question now is whether solving GNTK regression is equivalent to training an infinite-wide multi-layer GNN using gradient descent. In this paper, we provide three new theoretical results. First, we formally prove this equivalence for graph-level regression. Second, we present the first GNTK formulation for node-level regression. Finally, we prove the equivalence for node-level regression.

  \end{abstract}
  \thispagestyle{empty}
\end{titlepage}

{
}
\newpage

\else
\maketitle
\begin{abstract}

\end{abstract}

\fi

\section{Introduction}

Many deep learning tasks need to deal with graph data, including social networks~\cite{yzw20}, bio-informatics~\cite{zl17,ywh20}, recommendation systems~\cite{yhc18}, and autonomous driving~\cite{wwm20, ysl20}. Due to the importance of these tasks, people turned to Graph Neural Networks (GNNs) as the de facto method for machine learning on graph data. GNNs show SOTA results on many graph-based learning jobs compared to using hand-crafted combinatorial features. As a result, GNNs have achieved convincing performance on a large number of tasks on graph-structured data.

Today, one major trend in theoretical deep learning is to understand when and how deep learning works through Neural Tangent Kernel (NTK)~\cite{jgh18}. NTK is a pure kernel method, and it has been shown that solving NTK is equivalent to training a fully-connected infinite-wide neural network whose layers are trained by gradient descent~\cite{adhlsw19,lsswy20}.
This is a major step forward in theoretical deep learning because researchers can now use traditional mathematical tools to analyze properties in deep neural networks and to explain various neural network optimizations from a theoretical perspective. For example, \cite{jgh18} uses NTK to explain why early stopping works in training neural networks. 
\cite{als19a, als19b} build general toolkits to analyze multi-layer networks
with ReLU activations and prove why stochastic gradient descent (SGD) can find global minima on the training objective of DNNs in polynomial time.


A natural next step is thus to apply NTK to GNNs, called \textit{Graph Neural Tangent Kernel (GNTK)}. \cite{dhs+19} provides the formulation for graph-level regression and demonstrates that GNTK can empirically achieve similar accuracy as GNNs in various bioinformatics datasets. 

This raises an important \textit{theoretical} question:
\begin{center}
   {\it Is solving  Graph Neural Tangent Kernel regression equivalent to using gradient descent to train an infinitely wide graph neural network?}
\end{center}

\begin{figure}[!ht]
    \centering
    \includegraphics[width=0.9\textwidth]
    {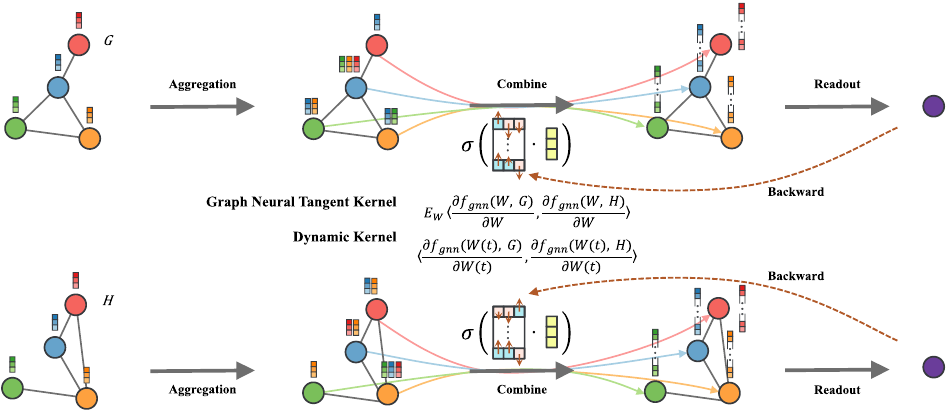}
    \caption{
    Our definitions for Graph Neural Tangent Kernel(GNTK) and Graph Dynamic Kernel. 
    We define $W$ as the parameters of GNN. 
    When we load parameters $W$ from $W(t)$ the parameters of GNN optimizing with gradient descent at iteration $t$, then the inner product of the gradients on the GNN parameters is the Graph Dynamic Kernel $ \k_{t}(G,H)$.
    When we load parameters $W$ from samples over a normal distribution, then the expectation taken on the inner product of gradients on the GNN parameters is GNTK $\k_{\gntk}(G,H)$. 
    }
    \label{fig:0}
\end{figure}

This question is worthwhile because GNTK can potentially enable researchers to understand why and how GNNs work, as such in neural network training through analysis of the corresponding NTK. At the same time, this question is more challenging than the equivalence result of NTK, because we need deal with the structures of the graphs.

In this paper, we provide an affirmative answer. We start from the formulation of graph-level regression in \cite{dhs+19}. We provide the first formal proof for the equivalence.
To achieve this result, we develop new proof techniques, including iterative Graph Tangent Kernel regression and Graph Dynamic Kernel. The high-level idea is that we use an iterative regression to bridge the connection between the neural network and the kernel method.

We then consider node-level regression, another graph learning task which has many real-world applications, such as social networks. We provide the first GNTK formulation for node-level regression. We also establish the equivalence result for node-level regression.

Our paper makes the following contributions:

\begin{itemize}
    \item We present the first graph neural tangent kernel formation for node-level regression.
    \item We prove the equivalence of graph neural tangent kernel and graph neural network training for both graph-level and node-level regression.
\end{itemize}

\section{Related Work}
\label{sec:related}

\paragraph{Kernel methods}
Instead of learning some fixed set of parameters based on a set of features, kernel methods can ``remember'' training examples in their weights~\cite{dzps19,ll18,als19a,als19b}. After learning on the set of inputs, the kernel function $\k$ is generated, and the kernel function can be used to predict unlabeled inputs. Our work is related to the recent breakthroughs in drawing the connections between kernel methods and neural networks~\cite{dfs16,d17,jgh18,cb18}. 
To the best of our knowledge, our work is the first that draws the theoretical connections between kernel methods and deep learning in the context of graph neural networks.



\paragraph{Convergence of neural network}

Many works study the convergence of deep neural network~\cite{bg17,t17,zsjbd17,s17,ly17,zsd17,dltps18,glm18,brw19} with random input assumptions. Recently, instead of assuming random input, several works study the convergence of neural networks in the over-parametrization regime~\cite{ll18,dzps19,als19a,als19b,dllwz19,adhlw19,adhlsw19,sy19,lsswy20,bpsw21,syz21,szz21}. These works only need a weak assumption on the input data such that the input data is separable: the distance between any two input data is sufficiently large. These works prove that if the neural network's width is sufficiently large, a neural network can converge. We use the convergence result in our proofs to bound the test predictor of graph neural network training.

\paragraph{Graph neural network (GNNs) and graph neural tangent kernel (GNTK)}

Graph neural networks are deep learning methods applied to graphs~\cite{zhang2020deep,wu2020comprehensive,chami2020machine,zhou2020graph, xwwf22, qwf22, xwwqzf21}. They have delivered promising results in many areas of AI, including social networks~\cite{yzw20, wu2020graph}, bio-informatics~\cite{zl17,ywh20, fout2017protein}, knowledge graphs \cite{hamaguchi2017knowledge}, recommendation systems~\cite{yhc18}, and autonomous driving~\cite{wwm20, ysl20}.
Due to their convincing performance, GNNs have become a widely applied graph learning method.

Graph neural tangent kernel applies the idea of the neural tangent kernel to GNNs. \cite{dhs+19} provides the first formulation for graph-level regression, and it shows that the kernel method can deliver similar accuracy as training graph neural networks on several bio-informatics datasets. However, \cite{dhs+19} does not provide a formal proof, and it does not consider node-level regression.

\paragraph{Sketching}
Sketching is a well-known technique to improve performance or memory complexity~\cite{cw13}. It has wide applications in linear algebra, such as linear regression and low-rank approximation\cite{cw13,nn13,mm13,rsw16,swz17,hlw17,alszz18,swz19_neurips1,swz19_neurips2,djssw19,gsy23}, training over-parameterized neural network~\cite{syz21, szz21, zhasks21}, empirical risk minimization~\cite{lsz19, qszz23}, linear programming \cite{lsz19,jswz21,sy21}, distributed problems \cite{wz16,bwz16, jll+20, jll+21}, clustering~\cite{emz21}, generative adversarial networks~\cite{xzz18}, kernel
density estimation~\cite{qrs+22},
tensor decomposition \cite{swz19_soda}, trace estimation~\cite{jpwz21}, projected gradient descent~\cite{hmr18, xss21}, matrix sensing~\cite{dls23_sensing, qsz23, gsyz23}, softmax regression \cite{as23,lsz23,dls23,gsy23,ssz23}, John Ellipsoid computation \cite{ccly19,syyz22}, semi-definite programming \cite{gs22}, kernel methods~\cite{acw17, akkpvwz20, cy21, swyz21, gsz23}, adversarial training~\cite{gqsw22}, anomaly detection~\cite{wglf23, wqbxmf22}, cutting plane method \cite{jlsw20}, discrepany \cite{z22}, federated learning~\cite{rpuisbga20},  reinforcement learning~\cite{akl17, wzd+20,ssx21},  relational database \cite{qjs+22}.

\section{Graph Neural Tangent Kernel}
\label{sec:result}


We first consider the graph regression task. Before we continuous, we introduce some useful notations as follows: 
 Let $G = (U, E)$ be a graph with node set $U$ and edge set $E$. A graph contains $|U| = N$ nodes. The dimension of the feature vector of each input node is $d$.  $\mathcal{N}$ is the neighboring matrix that represents the topology of $G$. Each node $u \in U$ has a feature vector $h(G)_{u} \in \R^d$. $\mathcal{N}(u)$ represents all the neighboring nodes of $u$.  A dataset contains $n$ graphs. We denote training data graph $\G = \{ G_1, \cdots, G_n \}$ with $G_i\in\R^{d \times N}~i\in[n]$. We denote label vector $Y\in\R^n$ corresponding to $\G$. Ours goal is to learn to predict the label of the graph $G_{\test}$.  For presentation simplicity, let's define an ordering of all the nodes in a graph. We can thus use $u_i$ to be the $i$th node in the graph. 
  We use $[k]$ to denote the set $\{1, 2, \cdots , k\}$.

\subsection{Graph Neural Network}


For simplicity, we only consider a single-level, single-layer GNN. A graph neural network (GNN) consists of three operations: {\aggregate}, {\combine} and {\readout}. 

{\bf {\aggregate}.}
The {\aggregate} operation aggregates the information from the neighboring nodes. We define the sum of feature vectors of neighboring nodes as 
    $h_{G,u} := \sum_{ v \in \mathcal{N}(u) } h(G)_v$.

Let $R:=\max_{i \in [n] } \max_{u \in G_i} h(G)_u$, we have that
 $ \| \sum_{v \in {\cal N}(u)} h(G_i)_v  \|_2
\leq \sum_{v \in {\cal N}(u)} \|  h(G_i)_u \|_2 
\leq R\cdot |{\cal N}(u)|$
where the first step follows from triangle inequality, the second step follows from the definition of $R$.
As a result, 
$\max_{i \in [n] } \max_{u \in G_i}  \| \sum_{v \in {\cal N}(u)} h(G_i)_v  \|_2 \leq R\max_{i \in [n] } \max_{u \in G_i}|{\cal N}(u)|  $. 
When graphs in $\G$ are sparse, the upper bound in the above equation will be $\alpha R$ where $\alpha$ is a constant and satisfy $\alpha \ll N$. When the graphs are dense, the upper bound above will be $N R $.

{\bf {\combine}.}
The {\combine} operation has a fully connected layer with ReLU activation. $W=[w_1,\cdots,w_m]\in \R^{d \times m}$ represents the weights in the neural network. $\sigma$ is the ReLU function ($\sigma(z) = \max\{0,z\}$), and $m$ represents the output dimension of the neural network. This layer can be formulated as $ h'_{G,u} := \frac{1}{\sqrt{m}}  \cdot \sigma(W^\top h_{G,u}) \in \R^m.$

{\bf {\readout}.}
The final output of the GNN on graph $G$ is generated by the {\readout} operation. We take summation over the output feature with weights $a_r \in \{-1, +1\},~r\in[m]$ or $a = [a_1, \cdots, a_m]^\top$, which is a standard setting in the literature~\cite{dzps19,sy19, bpsw21}. This layer can be formulated as $f_{\gnn}(G) := \sum_{r=1}^m a_r \sum_{u \in G} [h'_{G,u}]_r \in \R. $

Combining the previous three operations, we now have the definition of a graph neural network function.

\begin{definition}[Graph neural network function]\label{def_main:f_nn}
We define a graph neural network 
$f_{\gnn}(W, a, G): \R^{d\times m}\times \R^m \times \R^{d \times N}\rightarrow\R$
as the following form
$f_{\gnn} (W, a, G) = \frac{1}{\sqrt{m}} \sum_{r=1}^m a_r \sum_{l =1}^N \sigma(w_r^\top \sum_{v\in {\cal N}(u_l)} h(G)_v)$.

Furthermore, We denote 
$f_{\gnn}(W, a, \G) = [f_{\gnn}(W, a,G_1),\cdots, f_{\gnn}(W,a, G_n)]^\top$
to represent the GNN output on the training set.
\end{definition}
The normalization $1/\sqrt{m}$ is a standard formulation for GNN~\cite{dhs+19}. Similar to other works in theoretical deep learning~\cite{ll18,als19a, dzps19,sy19, bpsw21,szz21,als+22,hswz22}, we consider only training $W$ while fixing $a$, so we can write $f_{\gnn}(W,G) = f_{\gnn}(W, a, G)$.

\subsection{Graph Neural Tangent Kernel}
We now translate the GNN architecture to its corresponding Graph Neural Tangent Kernel (GNTK). 


\begin{definition}[Graph neural tangent kernel]\label{def_main:ntk_phi}
We define the graph neural tangent kernel (GNTK) corresponding to the graph neural networks $f_{\gnn}$ as following 
\begin{align*} 
	\k_{\gntk}(G, H) = \E_W \left[\left\langle \frac{\partial f_{\gnn}(W,G)}{\partial W},\frac{\partial f_{\gnn}(W,H)}{\partial W} \right\rangle \right],
\end{align*} 
where
\begin{itemize}
    \item $G,H$ are any input graph,
    \item and the expectation is taking over
 \begin{align*}
    w_r\overset{i.i.d.}{\sim} \N(0,I),~r=1, \cdots, m .
 \end{align*}
\end{itemize}

We define $H^{\cts}\in\R^{n\times n}$ as the kernel matrix between training data as $[H^{\cts}]_{i,j} = \k_{\gntk}(G_i, G_j) \in \R$. We denote the smallest eigenvalue of $H^{\cts}$ as $\Lambda_0$.
\end{definition}
In Figure \ref{fig:0},
we provide an illustration on GNTK. 
We sample the parameters from a normal distribution and then load the sampled parameters into the GNN. We use those loaded parameters to perform the forward pass on the graph and then calculate the gradient. We forward each pair of the graph in the dataset $\G$ to the same GNN with the same parameters and calculate the gradient on $W$ the parameters of GNN.
Finally, taking expectations on the inner product of the gradient on each pair of the graph, we can obtain GNTK. 
GNTK is a constant and does not change in the GNN training process.


We define the corresponding feature function as follows:
\begin{definition}[Feature function]
Let $\mathcal{F} $ be a Hilbert space.
\begin{itemize}
    \item We denote the feature function corresponding to the kernel $\k_{\gntk}$ as $\Phi:\R^{d\times N} \rightarrow \mathcal{F}$, which satisfies
	$\langle\Phi(G),\Phi(H)\rangle_\mathcal{F} = \k_{\gntk}(G,H)$,for any graph data $G$, $H$. 
 \item We write $\Phi(\G)=[\Phi(G_1),\cdots,\Phi(G_n)]^\top$.
\end{itemize}


\end{definition}



Next, we provide the definition of Graph Neural Tangent Kernel regression. 

\begin{definition}[Graph Neural Tangent Kernel regression]\label{def_main:krr_ntk}

We consider the following neural tangent kernel regression problem:
$\min_{\beta} \frac{1}{2}\| Y - \kappa f_{\gntk}(\beta,\G) \|_2^2$
where
\begin{itemize}
    \item $\beta\in\mathcal{F}$ is weight of kernel regression,
    \item  $f_{\gntk}(\beta,\G) = \Phi(\G) \beta \in \R$ denotes the prediction function.
\end{itemize}
\end{definition}

\subsection{Notations for Node-Level regression}
Here, we consider the node level regression task. In our setting, our dataset contains a single graph $G=(V,E)$ where $V$ denotes the set of nodes in graph $G$, and $E$ denotes the set of edges in graph $G$. Suppose that $|V|=N$, the dataset also contains labels $Y=(y_1,\cdots,y_N)$ with $y_i$ corresponding to the label of the node $u_i, ~i\in[N]$ in the graph. 

In this paper, we consider two different setting for node level regression training, inductive training and transductive training. In the transductive training process, we can observe the whole $G=(V,E)$. We can observe the labels of nodes in $V_{\train}$, but we cannot observe the labels of nodes in $V_{\test}$. In the inductive training process, we partition $V$ into two disjoint sets $V_{\train}$ and $V_{\test}$. Let $E_{\train}=\{(u,v)\in E|u\in V_{\train} \wedge v\in V_{\train}\}$. Let $G_{\train}=(V_{\train}, E_{\train})$. We can only observe the subgraph $G_{\train}=(V_{\train},E_{\train})$ and feature in $V_{\train}$. During the test process, we can observe the whole graph $G=(V,E)$ in both the inductive and transductive setting, but we cannot observe any label. We prove the equivalence between GNN and GNTK for both of the settings. For simplicity, we only consider transductive setting here. We left the inductive setting into appendix.

In the rest of this section, we will provide the definitions of GNN and GNTK for node level regression and the definition of their corresponding training process. 

We first show the definition of GNN in node level. 
\begin{definition}[Graph neural network function (node level)]\label{def_main:f_nn_node}
Given the graph $G$, we define a graph neural network $f_{\gnn,\node}(W, a, u): \R^{d\times m}\times \R^m \times [N]\rightarrow\R$ 
as the following form
\begin{align*} 
f_{\gnn,\node} (W, a, u) = \frac{1}{\sqrt{m}} \sum_{r=1}^m a_r \sigma(w_r^\top \sum_{v\in {\cal N}(u)} h(G)_v), 
\end{align*}  
Besides, We denote 
\begin{align*} 
 & ~f_{\gnn,\node}(W, a, G) 
=  [f_{\gnn,\node}(W, a, u_1),\cdots, f_{\gnn,\node}(W,a, u_N)]^\top \in \R^N
\end{align*}
to represent the GNN output on the training set. 
\end{definition}
Note that in node level, the definition of GNN do not contain {\readout} layer. There are only a single {\aggregate} layer and a single {\combine} layer. 

We now show the definition of GNTK in node level as follows:
\begin{definition}[Graph neural tangent kernel (node level)]\label{def_main:ntk_phi_node}
We define the graph neural tangent kernel (GNTK) corresponding to the graph neural networks $f_{\gnn, \node}$ as following 
\begin{align*} 
 & ~ \k_{\gntk, \node}(u, v) 
	=  \E_W \left[\left\langle \frac{\partial f_{\gnn, \node}(W,u)}{\partial W},\frac{\partial f_{\gnn, \node}(W,v)}{\partial W} \right\rangle \right],
\end{align*}
where
\begin{itemize}
    \item $u,v$ are any input node in graph $G$, 
    \item and the expectation is taking over
 \begin{align*}
    w_r\overset{i.i.d.}{\sim} \N(0,I),~r=1, \cdots, m.
\end{align*}
\end{itemize}

We define $H^{\cts,\node}\in\R^{N\times N}$ as the kernel matrix between training data as $[H^{\cts,\node}]_{i,j} = \k_{\gntk,\node}(u_i, u_j) \in \R$. We denote the smallest eigenvalue of $H^{\cts,\node}$ as $\Lambda_{0,\node}$.
\end{definition}
In graph level Definition \ref{def_main:ntk_phi}, any single term in the kernel is about two different graphs, while in the node level Definition \ref{def_main:ntk_phi_node}, the kernel is about two different nodes in a single graph. 

The training process of GNN in node level is similar to the training of GNN in graph level defined in Definition \ref{def:nn}. 
\begin{definition}[Training graph neural network (node level)]\label{def_main:nn_node}
Let $\kappa\in(0,1)$ be a small multiplier. We initialize the network as 
   $a_r\overset{i.i.d.}{\sim} \unif[\{-1,1\}], ~ w_r(0)\overset{i.i.d.}{\sim} \N(0,I_d) $
We consider solving the following optimization problem using gradient descent:
\begin{align*} 
\min_{W} \frac{1}{2}\| Y - \kappa f_{\gnn,\node}(W, G) \|_2  .
 \end{align*}

Furthermore, we define
\begin{itemize}
    \item $w_r(t),r\in[m]$ as the weight at iteration $t$,
    \item  training data predictor at iteration $t$ as
$u_{\gnn,\node}(t) = \kappa f_{\gnn,\node}(W(t), G)$,
\item  and test data predictor at iteration $t$ as 
\begin{align*} 
u_{\gnn,\node,\test}(t) = \kappa f_{\gnn,\node}(W(t), v_{\test}).
\end{align*}
\end{itemize}

\end{definition}

The regression of GNTK in node level is similar with the regression of GNTK in graph level defined in Definition \ref{def_main:krr_ntk}.
\begin{definition}[Graph Neural Tangent Kernel regression (node level)]\label{def_main:krr_ntk_node}

We consider the following neural tangent kernel regression problem:
 \begin{align*}
\min_{\beta} \frac{1}{2}\| Y - \kappa f_{\gntk}(\beta,G) \|_2^2
 \end{align*}
where 
\begin{itemize}
    \item $\beta\in\mathcal{F}$ is weight of kernel regression.
    \item  $f_{\gntk,\node}(\beta,G) = \Phi(G) \beta \in \R$ denotes the prediction function. 
\end{itemize}
\end{definition}

\section{Equivalence between GNTK and infinite-wide GNN}
\label{sec:main-result}


\subsection{Equivalence result for graph-level regression}

Our main result is the following equivalence between solving kernel regression and using gradient descent to train an infinitely wide single-layer graph neural network. It states that if the neural network width $m$ is wide enough and the training iterations $T$ is large enough then the test data predictor between Graph Neural Network $u_{\gnn,\test}(T)$ and the Graph Neural Tangent Kernel regression $u_{\test}^*$ can be less than any positive accuracy $\epsilon$.

In our result, we use the fact that inputs are bounded, which is a standard setting in the optimization field~\cite{lsswy20}. Let $R$ denote the parameter that bound the feature vectors $h(G)_u$: $\max_{i \in [n] } \max_{u \in G_i} \|  h(G_i)_u \|_2 \leq R $. This equation means that there is a bound on the maximum value of the input data. 



The equivalence result is as follows:
\begin{theorem}[Equivalence result for graph level regression]\label{thm_main:main_test_equivalence}
Given
\begin{itemize}
    \item training graph data $\G = \{ G_1, \cdots, G_n \}$,
    \item  corresponding label vector $Y \in \R^n$,
    \item arbitrary test data $G_{\test}$.
\end{itemize}
Let 
\begin{itemize}
    \item $T > 0$ be the total number of iterations, 
    \item $N$ be the number of nodes in a graph,
    \item  $u_{\gnn,\test}(t) \in \R$ and $u_{\test}^* \in \R$ be the test data predictors defined in Definition~\ref{def:nn} and Definition~\ref{def_main:krr_ntk} respectively.
\end{itemize}

For any accuracy $\epsilon \in (0,1/10)$ and failure probability $\delta \in (0,1/10)$, if the following conditions hold:
\begin{itemize}
    \item the multiplier $\kappa = \wt{O}( N^{-2}\poly(\epsilon, \Lambda_0, R^{-1}, n^{-1}))$,
    \item  the total iterations $T= \wt{O}(N^{4}\poly(\epsilon^{-1}, \Lambda_0^{-1}, R, n))$,
    \item  the width of the neural network $m \geq\wt{O}( N^2\poly(\epsilon^{-1}, \Lambda_0^{-1}, n, d))$,
    \item  and the smallest eigenvalue of GNTK $\Lambda_0 > 0$,
\end{itemize}
 then for any $G_{\test}$, with probability at least $1-\delta$ over the random initialization, we have
\begin{align*} 
\| u_{\gnn,\test}(T) - u_{\test}^* \|_2 \leq \epsilon.
\end{align*}
Here $\wt{O}(\cdot)$ hides $\poly\log(Nn/(\epsilon \delta \Lambda_0 ))$.
\end{theorem}
Note that when $\eps$ goes into $0$, the width of GNN will go into infinity, our result can deduce that the limit of $ \|u_{\gnn,\test}(T)-u^*_{\test}\|_2\rightarrow0$ when the width of GNN $m\rightarrow\infty$ and the training iteration $T\rightarrow\infty$. It is worth noting that our result is a non-asymptotic convergence result: we provide a lower bound on the training time $T$ and the width of layers $m$. 

The number of nodes $N$ has a non-linear impact on the required training iterations $T$ and the width of layer $m$. Our Theorem \ref{thm_main:main_test_equivalence} shows that $m$ depends on $N^{2}$ and $T$ depends on $N^{4}$.

Moreover, the number of vertices does not need to be fixed in our Theorem \ref{thm_main:main_test_equivalence}. In fact, for any graph that contains number of vertices less than $N$, we can add dummy vertices with $0$ feature until the graph contains $N$ vertices. Then, it reduces to the case where all the graphs in the dataset $\G$ contain $N$ vertices. 

Next, we will show the generalized $\delta$-separation assumption under which the condition $\Lambda > 0$ in Theorem \ref{thm_main:main_test_equivalence} will always hold.

\begin{theorem}[Spectral gap of shifted GNTK]
For graph neural network with shifted rectified linear unit (ReLU) activation, 
$f_{\gnn, \mathrm{shift}} (W, a, b, G) = \frac{1}{\sqrt{m}} \sum_{r=1}^m a_r \sum_{l =1}^N \sigma(w_r^\top \sum_{v\in {\cal N}(u_l)} h(G)_v-b)$, where $b\in\R$. We define $H^{\cts}_{ \mathrm{shift}}\in\R^{n\times n}$ as the shifted kernel matrix between training data: \begin{align*} 
& ~ [H^{\cts}_ \mathrm{shift}]_{i,j} 
=  \E_W \left[\left\langle \frac{\partial f_{\gnn, \mathrm{shift}}(W,G)}{\partial W},\frac{\partial f_{\gnn, \mathrm{shift}}(W,H)}{\partial W} \right\rangle \right],
\end{align*}

where 
\begin{itemize}
    \item $G,H$ are any input graph,
    \item  and the expectation is taking over $w_r\overset{i.i.d.}{\sim} \N(0,I),~r=1, \cdots, m $.
\end{itemize}

Let 
\begin{itemize}
    \item $ h_{G, u} = \sum_{ v \in \mathcal{N}(u) } h(G)_v, ~x_{i, p}=h_{G_i,u_p}, ~i\in[n], p\in[N] $.
    \item initialize $w_r\overset{i.i.d.}{\sim} \N(0,I),~r=1, \cdots, m$. 
\end{itemize}
Assuming that our dataset satisfy the $\delta$-separation: 

For any $i,j \in[n], ~p,q\in[N]$, 
\begin{align*}
\min\{\| \ov{x}_{i, p} - \ov{x}_{j,q} \|_2, \| \ov{x}_{i, p} + \ov{x}_{j,q} \|_2    \}  \leq \delta ,
\end{align*}
where $\ov{x}_{i,p}= x_{i,p}/\|x_{i,p}\|_2$. Then we claim that, for $\lambda:=\lambda_{\min} (H^{\cts}_{\mathrm{shift}})$, 
\begin{align*}
 \exp(-b^2/2) \geq \lambda \geq \exp(-b^2/2)\cdot  \poly(\delta, n^{-1}, N^{-1})
\end{align*}
\end{theorem}
In this theorem, we consider shifted GNTK, which is motivated by the shifted NTK \cite{syz21} and is more general than the standard GNTK definition \cite{dhs+19}. Here, we expand $\delta$-separation into the graph literature. The theorem says that if $\delta$-separation is satisfied for the graphs, then the Graph Neural Tangent Kernel will be positive definite. More specifically, in the theorem above, when we set $b=0$, we have that $\Lambda_0 = \lambda\geq  \poly(\delta, n^{-1}, N^{-1})$. 

\subsection{Equivalence result for node-level regression}
We also prove the equivalence result for node-level regression setting. The result is as follows:
\begin{theorem}[Equivalence result for node level regression, informal version of Theorem~\ref{thm:main_test_equivalence_node}]\label{thm_main:main_test_equivalence_node}
Given 
\begin{itemize}
    \item training graph data $G$,
    \item corresponding label vector $Y \in \R^N$,
    \item arbitrary test data $v$. 
\end{itemize}
 Let 
 \begin{itemize}
     \item  $T > 0$ be the total number of iterations.
     \item $u_{\gnn,\node,\test}(t) \in \R$ and $u_{\test,\node}^* \in \R$ be the test data predictors defined in Definition~\ref{def_main:nn_node} and Definition~\ref{def_main:krr_ntk} respectively. 
 \end{itemize}

For any accuracy $\epsilon \in (0,1/10)$ and failure probability $\delta \in (0,1/10)$, if 
\begin{itemize}
    \item the multiplier $\kappa = \wt{O}(N^{-1}\poly(\epsilon, \Lambda_{0,\node}, R^{-1}))$,
    \item  the total iterations $T=\wt{O}(N^{2}\poly(\epsilon^{-1}, \Lambda_{0,\node}^{-1}))$, 
    \item the width of the neural network $ m \geq \wt{O}(N^{10}\poly(\epsilon^{-1}, \Lambda_{0,\node}^{-1}, d)) $,
    \item  and the smallest eigenvalue of node level GNTK $\Lambda_{0,\node} > 0$, 
\end{itemize}
then for any $v_{\test}$, with probability at least $1-\delta$ over the random initialization, we have
\begin{align*} 
\| u_{\gnn,\test,\node}(T) - u_{\test,\node}^* \|_2 \leq \epsilon.
\end{align*}
Here $\wt{O}(\cdot)$ hides $\poly\log(N/(\epsilon \delta \Lambda_{0,\node} ))$. 
\end{theorem}
Our result shows that the width of the neural network $m$ depends on $N^{10}$, where $N$ is the number of nodes in the graph. The total training iteration $T$ depends on $N^{2}$.

\section{Overview of Techniques}
\label{sec:techniques}


\paragraph{High-level Approach}


Given any test graph $G_{\test}$, our goal is to bound the distance between the prediction of graph kernel regression $u_{\test}^*$ and the prediction of the GNN $u_{\gnn,\test}(T)$ at training iterations $t=T$. 
Our main approach is that we consider an iterative GNTK regression process $u_{\gntk,\test}(t)$ that is optimized using gradient descent.


Solving the GNTK kernel regression is a single-step process, while training GNN involves multiple iterations. Here, our proposed iterative GNTK regression $u_{\gntk,\test}(t)$, which is optimized by gradient descent,  bridges the gap. We prove that when our iterative GNTK regression finishes, the final prediction on $G_{\test}$ is exactly the same as if solving GNTK as a single-step process. 
Thus, to prove GNTK=GNN, the remaining step is to prove the equivalence of iterative GNTK regression and the GNN training process. 


More specifically, We use the iterative GNTK regression's test data predictor $u_{\gntk,\test}(T)$ at training iterations $T$ as a bridge to connect GNTK regression predictor $u_{\test}^*$ and GNN predictor $u_{\gnn,\test}(T)$.
We conclude that there are three major steps to prove our main result.
\begin{itemize}
    \item By picking a sufficiently large number of training iteration $T$, we bound the test predictor difference for GNTK regression at training iterations $T$ and the test error after GNTK converges $|u_{\gntk,\test}(T) - u_{\test}^*| \leq \epsilon / 2$ by any accuracy $\epsilon > 0$. 
    \item After fixing the number of training iteration $T$ in Step 1, we bound the test predictor difference between GNN and GNTK $|u_{\gnn,\test}(T) - u_{\gntk,\test}(T)| \leq \epsilon / 2$ at training iterations $T$.
    \item Lastly, we combine the results of step 1 and step 2 using triangle inequality. This finishes the proof of the equivalence between training GNN and GNTK kernel regression, i.e., $|u_{\gnn,\test}(T) - u_{\test}^*| \leq \epsilon$.
\end{itemize}
\vspace{-1.5mm}
For the first part, the high level idea is that $u_{\test}^*$ is $u_{\gntk, \test}(t)$ when $t\to\infty$, we then obtain the desired result by the linear convergence of the gradient flow of $u_{\gntk, \test}(t)$. The second part needs sophisticated tricks. The high-level idea is that we can split the second part into two sub-problem by using the Newton-Leibniz formula and the almost invariant property of Graph Dynamic Kernel. The third part follows from triangle inequality.




\paragraph{Iterative Graph Neural Tangent Kernel regression}
Iterative Graph Neural Tangent Kernel regression provides a powerful theoretical tool to compare the GNN training process with GNTK regression. 
We consider solving the following optimization problem using gradient descent with $\beta(0)$ initialized as $0$: 
\begin{align*}
  \min_{\beta} \frac{1}{2}\| Y - \kappa f_{\gntk}(\beta,\G) \|_2^2,
\end{align*} 
where $f_{\gntk}(\beta,\G) $ is defined in Definition \ref{def_main:krr_ntk}. At iteration $t$, we define the weight as $\beta(t)$, the training data predictor $u_{\gntk}(t) = \kappa\Phi(\G)\beta(t) \in \R^n$, the test data predictor $u_{\gntk,\test}(t) = \kappa\Phi(G_{\test})^\top\beta(t) \in \R$. 

\paragraph{Bridge GNTK regression and iterative GNTK regression} Then, we claim the relationship between GNTK regression and iterative GNTK regression. When the iteration goes into infinity, we get the final training data predictor  $u^* = \lim_{t \to \infty} u_{\gntk}(t) $ and the final test data predictor as $	u_{\test}^* = \lim_{t\to\infty} u_{\gntk,\test}(t)$. Due to the strongly convexity~\cite{lsswy20}, the gradient flow converges to the optimal solution of the objective $\| Y - \kappa f_{\gntk}(\beta,\G) \|_2^2$, 
therefore, the optimal parameters obtained from GNTK regression 
\begin{align*} 
\beta^* =  \lim_{t\to\infty} \beta(t) 
=   \kappa^{-1} ( \Phi(\G)^\top\Phi(\G) )^{-1} \Phi(\G)^\top Y
,
\end{align*}
which indicate that the optimal parameters $\beta^*$ is exactly $\lim_{t\to\infty} \beta(t)$ the parameters of iterative GNTK regression, when training finish.

\paragraph{Graph Dynamic Kernel}
Before we continue, we introduce the Graph Dynamic Kernel for the GNN training process. Graph Dynamic Kernel can help us analyze the dynamic of the GNN training process. For any data $G,H\in\R^{d \times N}$, we define $\k_t(G,H)\in\R$ as
\begin{align*} 
    \k_t(G,H)
    = \left\langle \frac{\partial f_{\gnn}(W(t),G)}{\partial W(t)},\frac{\partial f_{\gnn}(W(t),H)}{\partial W(t)} \right\rangle.
\end{align*}
where $W(t) \in \R^{d \times m}$ is the parameters of the neural network at iterations $t$ as defined in Definition~\ref{def:nn}. Similar to tangent kernel,  we define $\H(t) \in\R^{n\times n}$ as 
\begin{align*} 
[\H(t)]_{i,j} = \k_{t}(G_i, G_j)\in\R.
\end{align*}

As in Figure \ref{fig:0}, this concept have some common points with GNTK defined in Definition \ref{def_main:ntk_phi}. Both of them load the parameters $W$. Both of them calculate the inner product on a pair of gradients. However, Graph Dynamic Kernel has a few differences from GNTK. First, GNTK is an expectation taken over a normal distribution, while Graph Dynamic Kernel is not an expectation. Second, GNTK is a constant during the training process, while Graph Dynamic Kernel changes with training iteration $t$. Third, the Graph Dynamic Kernel loads parameters from the GNN training process, while GNTK loads parameters from samples of a normal distribution.

Graph Dynamic Kernel is very different than the Dynamic Kernel defined in NTK literature \cite{dzps19, sy19, lsswy20, syz21}. Our key results is that Graph Dynamic Kernel is almost invariant during the training process, which means Graph Dynamic Kernel becomes a constant when the width of the neural network goes into infinity. Although people already know that Dynamic Kernel has similar properties \cite{adhlsw19, lsswy20}, no previous works show whether the almost invariant property still holds for Graph Dynamic Kernel. In fact, GNNs have \textsc{Combine} layer and \textsc{Aggregate} layer, while NNs do not have those layer. Whether or not those layers influence the almost invariant property are not trivial. In fact, the perturbation on the Graph Dynamic Kernel during training process polynomial depends on $N$, where $N$ is the number of nodes in a graph. As a result, for graph neural network, the width should be chosen larger to ensure that it behaves like a GNTK.

\paragraph{Linear convergence}

We first show that GNNs and iterative GNTK are linear convergence. The linear convergence can be proved by calculating the gradient flow of GNTK regression $u_{\gntk,\test}(t)$ and GNN training $u_{\gnn,\test}(t)$ with respect to the training iterations $t$. Based on the gradient, we can obtain the desired results: 
\begin{align*}
\frac{\d \|u(t)-u^*\|_2^2}{\d t} \le - 2(\kappa^2 \Lambda_0) \|u(t)-u^*\|_2^2.
\end{align*} 
where $u(t)$ can be both $u_{\gnn}(t)$ and $u_{\gntk}(t)$. 

Then, we show that the test predictor difference between  GNTK regression and iterative GNTK regression is sufficiently small, when $T$ is sufficiently large. The linear convergence indicates that both $ u_{\gnn}(t)$ and $u_{\gntk}(t)$ converge at a significant speed. In fact, the linear convergence imply that the squared distance between $u(t)$ and $u^*$ converge to $0$ in exponential.
Based on the convergence result, we immediately get 
\begin{align*} 
|u_{\gntk,\test}(T) - u_{\test}^*| < \epsilon / 2
\end{align*}
by picking sufficiently large $T $.

\paragraph{Bridge iterative GNTK regression and GNN training process}

In this section, we will bound $|u_{\gnn,\test}(T)- u_{\gntk,\test}(T)|$. The high-level approach is that we can use the Newton-Leibniz formula to separate the problem into two parts:  First, we  bound the initialization perturbation $|u_{\gnn,\test}(0)- u_{\gntk,\test}(0)|$. Then, we bound 
\begin{align*} 
|\d( u_{\gnn,\test}(t)- u_{\gntk,\test}(t))/\d t |
\end{align*}
the change rate of the perturbation.  

For the first part, note that $u_{\gntk,\test}(0)=0$, we bound $|u_{\gnn,\test}(0)|$ by the concentration property of Gaussian initialization. 
For the second part, since we can prove that the gradient flow of the predictor of GNN during the GNN training process is associated with $\H(t)$ and the gradient flow of the predictor of iterative GNTK regression is associated with GNTK $\H^{cts}$. So bounding $\|H(t)-H^{\cts}\|_2$ can provide us with a bound on 
$
|\d( u_{\gnn,\test}(t)- u_{\gntk,\test}(t))/\d t |.
$ 
Note that $H^{\cts}$ is exactly the expectation of $H(0)$, where the expectation is taken on the random initialization of GNN. We bridge $H(t)$ and $H^{\cts}$ by the initial value of Graph Dynamic Kernel $H(0)$. To bound $\| \H(0)- \H^{\cts} \|_2$, we use the fact that $\E_{W(0)}[\H(0)]= \H^{\cts}$ and Hoeffding inequality. 
To bound $\| \H(0)- H(t) \|_2$, we use the almost invariant property of the Graph Dynamic Kernel.



For the node level equivalence, given any test node $v_{\test}$, the goal is to bound the distance between the node level graph kernel predictor $u_{\node,\test}^*$ and the node level GNN predictor $u_{\gnn,\node,\test}(T)$.
To prove the node level equivalence, we use a toolbox similar to the one of graph level equivalence. We introduce node-level iterative GNTK regression, which is optimized by gradient descent. Node-level iterative GNTK regression connects the node-level GNTK regression and node-level GNN training process. 
By leveraging the linear convergence of the gradient flow, and the almost invariant property of node-level Graph Dynamic Kernel, we get the equivalence result of node level. 

\paragraph{Spectral gap of shifted GNTK} 

To prove a lower bound on the smallest eigenvalue of the shifted GNTK, we first calculate the explicit formation of shifted GNTK, which turns out to be the summation of Hadamard product between covariance of the input graph dataset and the covariance of a indicator mapping. The indicator mapping takes $1$ when ReLU be activated by the input. We bound the eigenvalue of Hadamard product by analyzing the two covariance matrix separately. For the covariance of the input graph dataset, we need to bound the minimum magnitude of its diagonal entry, which can be further bounded by the norm of the input data. For the covariance of the indicator mapping, we need to bound the smallest eigenvalue, which is equal to the lower bound of the expectation of the inner product between the indicator mapping and any real vector. We bound the expectation by carefully decomposing the vector's maximum coordinate and the vector's other coordinates. To provide a good bound on the maximum coordinate, it turns out to be an anti-concentration problem of the Gaussian random matrix.

We delay more technique overview discussions to Section~\ref{sec:more_tech_overview} due to space limitation.

\section{Conclusion}
\label{sec:conclusion}
Neural Tangent Kernel~\cite{jgh18} is a useful tool in the theoretical deep learning to help understand how and why deep learning works. 
A natural extension of NTK for graph learning is \textit{Graph Neural Tangent Kernel (GNTK)}. \cite{dhs+19} has shown that GNTK can achieve similar accuracy as training graph neural networks for several bioinformatics datasets for graph-level regression. We provide the first GNTK formulation for node-level regression. We present the first formal proof that training a graph neural network is equivalent to solving GNTK for both graph-level and node-level regression. From our view, given that our work is theoretical, it doesn't lead to any negative societal impacts.



\newpage
\ifdefined\isarxiv
 
 \bibliographystyle{alpha}
 \bibliography{ref}
\else
    \bibliography{ref}
    \bibliographystyle{alpha}
\fi

\newpage 


\appendix
\onecolumn

\section*{Appendix}

\paragraph{Roadmap.}

In Section \ref{sec:notations}, we introduce several notations and tools that is useful in our paper.
In Section \ref{sec:more_tech_overview}, we provide more discussion of technique overview.
In  Section \ref{sec:formula_node_GNTK}, we give the formula for calculating node-level GNTK.
In Section \ref{sec:app_gnn_gntk_eq}, we prove the equivalence of GNN and GNTK for both graph level and node level regression. 
In Section \ref{sec:separation}, we provide a lower bound and a upper bound on the eigenvalue of the GNTK. In Section \ref{sec:node_gntk_formulation}, we deduce the formulation of node level GNTK.

\section{Notations and Probability Tools}
\label{sec:notations}

In this section we introduce some notations and sevreral probability tools we use in the proof.

\paragraph{Notations.}

 Let $G = (U, E)$ be a graph with node set $U$ and edge set $E$. A graph contains $|U| = N$ nodes. The dimension of the feature vector of each input node is $d$.  $\mathcal{N}$ is the neighboring matrix that represents the topology of $G$. Each node $u \in U$ has a feature vector $h(G)_{u} \in \R^d$. $\mathcal{N}(u)$ represents all the neighboring nodes of $u$.  A dataset contains $n$ graphs. We denote training data graph $\G = \{ G_1, \cdots, G_n \}$ with $G_i\in\R^{d \times N}~i\in[n]$. We denote label vector $Y\in\R^n$ corresponding to $\G$. For graph level regression, ours goal is to learn to predict the label of the graph $G_{\test}$. For node level regression, ours goal is to learn to predict the label of the node $v_{\test}$.  For presentation simplicity, let's define an ordering of all the nodes in a graph. We can thus use $u_i$ to be the $i$th node in the graph. 
We use $\sigma(\cdot)$ to denote the non-linear activation function. In this paper, $\sigma(\cdot)$ usually denotes the rectified linear unit (ReLU) function. We use $\dot{\sigma}(\cdot)$ to denote the derivative of $\sigma(\cdot)$.
We use $[k]$ to denote the set $\{1, 2, \cdots , k\}$.

We state Chernoff, Hoeffding and Bernstein inequalities.

\begin{lemma}[Hoeffding bound \cite{h63}]\label{lem:hoeffding}
Let $X_1, \cdots, X_n$ denote $n$ independent bounded variables in $[a_i,b_i]$. Let $X= \sum_{i=1}^n X_i$, then we have
\begin{align*}
\Pr[ | X - \E[X] | \geq t ] \leq 2\exp \left( - \frac{2t^2}{ \sum_{i=1}^n (b_i - a_i)^2 } \right).
\end{align*}
\end{lemma}

\begin{lemma}[Bernstein inequality \cite{b24}]\label{lem:bernstein}
Let $X_1, \cdots, X_n$ be independent zero-mean random variables. Suppose that $|X_i| \leq M$ almost surely, for all $i$. Then, for all positive $t$,
\begin{align*}
\Pr \left[ \sum_{i=1}^n X_i > t \right] \leq \exp \left( - \frac{ t^2/2 }{ \sum_{j=1}^n \E[X_j^2]  + M t /3 } \right).
\end{align*}
\end{lemma}

We state three inequalities for Gaussian random variables.

\begin{lemma}[Gaussian tail bounds]\label{lem:gaussian_tail}
Let $X\sim\N(\mu,\sigma^2)$ be a Gaussian random variable with mean $\mu$ and variance $\sigma^2$. Then for all $t \geq 0$, we have
\begin{align*}
  \Pr[|X-\mu|\geq t]\leq 2e^{-\frac{-t^2}{2\sigma^2}}.
\end{align*}
\end{lemma}

\begin{claim}[Theorem 3.1 in \cite{ls01}]\label{clm:gaussain_anti_shift}
Let $b>0$ and $r>0$. Then,
\begin{align*}
    \exp(-b^2/2)\Pr_{x\sim \N(0,1)}[|x|\leq r] \leq ~ \Pr_{x\sim \N(0,1)}[|x-b|\leq r] \leq ~ \Pr_{x\sim \N(0,1)}[|x|\leq r].
\end{align*}
\end{claim}

\begin{lemma}[Anti-concentration of Gaussian distribution]\label{lem:anti_gaussian}
Let $Z \sim {\N}(0,\sigma^2)$.
Then, for $t>0$,
\begin{align*}
    \Pr[|Z|\leq t]\leq \frac{2t}{\sqrt{2\pi}\sigma}.
\end{align*}
\end{lemma}

\section{More Discussion of Technique Overview}\label{sec:more_tech_overview}

\paragraph{Node level techniques}
The formulation of node level GNTK in Section \ref{sec:formula_node_GNTK} is deduced by the Gaussian processes nature of infinite width neural network. In the infinite width limit, the pre-activations $W^{(l)}_{(r)}f^{(l)}_{(r-1)}(u)$ at every hidden layer $r$, every level $l$ have all its coordinates tending to i.i.d. centered Gaussian processes. 
Note that, for the activation layer, if the covariance matrix and distribution of pre-activations is explicitly known, we can compute the covariance matrix of the activated values exactly. As a result, we can recursively compute the covariance of the centered Gaussian processes $\bm{\Sigma}^{(l)}_{(r)}(u,u')$. By computing the partial derivative of GNN output $f^{(L)}_{(R)}$ with respect to weight $W^{(l)}_{(r)}$, we obtain the explicit formulation for computing node level GNTK.

\section{Calculating Node-Level GNTK}
\label{sec:formula_node_GNTK}
In this section, we give the formula for calculating node-level GNTK. We consider multi-level multi-layer node level Graph Neural Tangent Kernels in this section. 

Let $f^{(l)}_{(r)}$ define the output of the Graph Neural Network in the $l$-th level and $r$-th layer. Suppose that there are $L$ levels and in each level, there are $R$ layers. At the beginning of each layer, we aggregate features over the neighbors $\neighbor(u)$ of each node $u$. More specifically, the definition of {\aggregate} layer is as follows: 
$
    f^{(1)}_{(1)}(u) =   \sum_{v \in \neighbor(u) } h(G)_v,
    $  and
    $
    f^{(l)}_{(1)}(u) =   \sum_{v \in \neighbor(u) }f^{(l-1)}_{(R)}(v), ~l\in \{2,3,\cdots, L\}.
$
In each level, we then perform $R$ fully-connected layers on the output of the {\aggregate} layer, which are also called {\combine} layer. The definition of {\combine} layer can be formulated as:
$
    f^{(l)}_{(r)}(u) = \sqrt{\frac{c_\sigma }{ m}}\sigma(W^{(l)}_{(r)} f^{(l)}_{(r-1)}(u)), ~l\in[L], r\in[R]
$ 
where $c_\sigma$ is a scaling factor, $W^{(l)}_{(r)}$ is the weight of the neural network in $l$-th level and $r$-th layer, $m$ is the width of the layer. In this paper we set $c_\sigma$ the same as \cite{hzr15}, i.e. $c_\sigma=2$. Let $f^{(L)}_{(R)}$ be the output of the Graph Neural Network. 

We show how to compute Graph Neural Tangent Kernel $\k_{\gntk,\node}(u, u')=\k_{(R)}^{(L)}(u,u')$ as follows. Let $\bm{\Sigma}^{(l)}_{(r)}(u,u')$ be defined as the covariance matrix of the Gaussian process of the  pre-activations $W^{(l)}_{(r)} f^{(l)}_{(r-1)}(u)$. We provide the formula of the GNTK in a recursive way. First, we give the recursive formula for {\aggregate} layer. We initialize with
$
\bm{\Sigma}^{(1)}_{(1)}(u,u') =  h(G)_u^\top h(G)_{u'}.
$
For $l\in \{2,3,\cdots, L\}$, 
$
\bm{\Sigma}^{(l)}_{(1)}(u,u') =  \sum_{v \in \neighbor(u) }\sum_{v' \in \neighbor(u') } \bm{\Sigma}^{(l-1)}_{(R)}(v,v').
$

Then, we give the recursive formula for the {\combine} layer.  
\begin{align*}
\bm{\Lambda}^{(l)}_{(r)}(u,u') =& \begin{pmatrix}
\bm{\Sigma}_{(r-1)}^{(l)}(u,u)  & \bm{\Sigma}_{(r-1)}^{(l)}(u,u') \\
\bm{\Sigma}_{(r-1)}^{(l)}(u',u) & \bm{\Sigma}_{(r-1)}^{(l)}(u',u') 
\end{pmatrix} \in \mathbb{R}^{2 \times 2},\\
\bm{\Sigma}^{(l)}_{(r)}(u,u') = &c_{\sigma}\mathbb{E}_{(a,b) \sim \mathcal{N}\left(\bm{0},\bm{\Lambda}^{(l)}_{(r)}\left(u,u'\right)\right)}[\sigma(a)\sigma(b)], 
\\
\bm{\dot{\Sigma}}^{(l)}_{(r)}\left(u,u'\right) = &c_\sigma\mathbb{E}_{(a,b) \sim \mathcal{N}\left(\bm{0},\bm{\Lambda}^{(l)}_{(r)}\left(u,u'\right)\right) }[\dot{\sigma}(a)\dot{\sigma}(b)]. 
\end{align*}

Finally, we give the formula for calculating the graph neural tangent kernel. For each {\aggregate} layer, the kernel can be calculated as follows:
$
\k^{(1)}_{(1)}(u,u') =~ h(G)_u^\top h(G)_{u'},
$
and 
\begin{align*}
 \k^{(l)}_{(1)}(u,u') =~ \sum_{v \in \neighbor(u) }\sum_{v' \in \neighbor(u')} \k^{(l-1)}_{(R)}(v,v').
\end{align*}
For each {\combine} layer, the kernel can be calculated as follows: 
\begin{align*} 
\k_{(r)}^{(l)}(u,u') = \k_{(r-1)}^{(l)}(u,u') \dot{\bm{\Sigma}}^{(l)}_{(r)}\left(u,u'\right)  + \bm{\Sigma}^{(l)}_{(r)}\left(u,u'\right).
\end{align*}
 
The Graph neural tangent kernel $\k_{\gntk,\node}(u, u')$ is the output $\k_{(R)}^{(L)}(u,u')$. 


\section{Equivalence Results}\label{sec:app_gnn_gntk_eq}

In this section we prove the equivalence of GNN and GNTK when the width goes to infinity for one layer of GNN 
and there is exactly one $\textsc{Aggregate}$, one $\textsc{Combine}$, and one $\textsc{ReadOut}$ operation. 

In Section~\ref{sec:equiv_definition}, we present some useful definitions.
In Section~\ref{sec:equiv_gradient_flow}, we compute the gradient flow of iterative GNTK regression and GNN training process.
In Section~\ref{sec:equiv_bound_ntk_test_T_and_test_*}, we prove an upper bound for $| u_{\gntk,\test}(T) - u_{\test}^* |$. 
In Section~\ref{sec:equiv_splitting_nn_test_T_and_ntk_test_T_into_three}, we bound the initialization and kernel perturbation.
In Section~\ref{sec:equiv_concentration_results_kernels}, we present the concentration results for kernels.
In Section~\ref{sec:equiv_intialization}, we upper bound the initialization perturbation.
In Section~\ref{sec:equiv_kernel_perturbation}, we upper bound the kernel perturbation.
In Section~\ref{sec:equiv_bound_nn_test_T_and_ntk_test_T}, we connect iterative GNTK regression and GNN training process.
In Section~\ref{sec:equiv_main_test_equivalence} we present the main theorem.

\subsection{Definitions}\label{sec:equiv_definition}
In this section we first present some formal definitions.
\begin{definition}[Graph neural network function]\label{def:f_nn}
We define a graph neural networks with rectified linear unit (ReLU) activation as the following form 
\begin{align*}
f_{\gnn} (W, a, G) = \frac{1}{\sqrt{m}} \sum_{r=1}^m a_r \sum_{l =1}^N \sigma (w_r^\top x_l) \in \R,
\end{align*}
where
\begin{itemize}
    \item $x \in \R^{d \times N}$ (decided by $G$) is the input,
    \item $w_r \in \R^d,~r\in[m]$ is the weight vector of the first layer, 
    \item  $W = [w_1, \cdots, w_m]\in\R^{d \times m}$, $a_r \in \R,~r\in[m]$ is the output weight, 
    \item $a = [a_1, \cdots, a_m]^\top$ and $\sigma(\cdot)$ is the ReLU activation function: $\sigma(z) = \max\{0,z\}$.
\end{itemize}
 In this paper, we consider only training the first layer $W$ while fixing $a$. So we also write $f_{\gnn}(W,G) = f_{\gnn}(W, a, G)$. We denote $f_{\gnn}(W, \G) = [f_{\gnn}(W, G_1),\cdots, f_{\gnn}(W, G_n)]^\top\in\R^n$.
\end{definition}

The training process of the GNN is as follows:
\begin{definition}[Training graph neural network]\label{def:nn}
Given
\begin{itemize}
    \item training data graph $\G = \{ G_1, \cdots, G_n \}$ (which can be viewed as a tensor $X\in\R^{n\times d \times N}$ and say $X_1, \cdots, X_N \in \R^{d \times n}$),
    \item  and corresponding label vector $Y\in\R^n$ where $d$ represent the dimension of a feature of a node and $N$ denote the number of nodes in a graph and $n$ is the size of training set.
\end{itemize}
Let
\begin{itemize}
    \item $f_{\gnn}$ be defined as in Definition~\ref{def:f_nn}.
    \item $\kappa\in(0,1)$ be a small multiplier.
\end{itemize}

We initialize the network as $a_r\overset{i.i.d.}{\sim} \unif[\{-1,1\}]$ and $w_r(0)\overset{i.i.d.}{\sim} \N(0,I_d)$. Then we consider solving the following optimization problem using gradient descent:
\begin{align}\label{eq:nn}
\min_{W} \frac{1}{2}\| Y - \kappa f_{\gnn}(W, \G) \|_2 
\end{align}
We denote $w_r(t),r\in[m]$ as the variable at iteration $t$. We denote
\begin{align}\label{eq:nn_predict_train} 
u_{\gnn}(t) = & ~ \kappa f_{\gnn}(W(t), \G) \notag \\
= & ~ \frac{\kappa}{\sqrt{m}} \sum_{r=1}^m a_r \sum_{l=1}^N \sigma (w_r(t)^\top X_{l}) \in \R^n
\end{align}
as the training data predictor at iteration $t$. Given any test data $\{ x_{\test,l} \}_{ l \in [N] } \in \R^d$ (decided by graph $G_{\test}$), we denote 
\begin{align}\label{eq:nn_predict_test} 
u_{\gnn,\test}(t) = & ~ \kappa f_{\gnn}(W(t), G_{\test}) \notag \\
= & ~ \frac{\kappa}{\sqrt{m}} \sum_{r=1}^m a_r \sum_{l=1}^N \sigma (w_r(t)^\top x_{\test,l}) \in \R
\end{align}
as the test data predictor at iteration $t$.
\end{definition}
In the definition above, $\kappa$ is the scaling factor \cite{lsswy20, adhlsw19}.

\begin{definition}[Graph neural tangent kernel and feature function]\label{def:ntk_phi}
We define the graph neural tangent kernel(GNTK) and the feature function corresponding to the graph neural networks $f_{\gnn}$ defined in Definition~\ref{def:f_nn} as following 
\begin{align*}
	\k_{\gntk}(G, H) = \E \left[\left\langle \frac{\partial f_{\gnn}(W,G)}{\partial W},\frac{\partial f_{\gnn}(W,H)}{\partial W} \right\rangle \right]
\end{align*}
where 
\begin{itemize}
    \item $G,H$ are any input data,
    \item  and the expectation is taking over $w_r\overset{i.i.d.}{\sim} \N(0,I),~r=1, \cdots, m$.
\end{itemize}
 Given
 \begin{itemize}
     \item  training data matrix $\G = \{ G_1, \cdots, G_n \}$,
 \end{itemize}
 we define $H^{\cts}\in\R^{n\times n}$ as the kernel matrix between training data as
\begin{align*}
	[H^{\cts}]_{i,j} = \k_{\gntk}(G_i, G_j) \in \R.
\end{align*}
Let $\Lambda_0 > 0$ represent the smallest eigenvalue of $H^{\cts}$, under the assumption that $H^{\cts}$ is positive definite. Additionally, for any data point $z$ belonging to $\R^d$, the kernel between the test and training data is denoted by $\k_{\gntk}(G, \G)$ and is a member of $\R^n$. We write it as:
\begin{align*}
	\k_{\gntk}(G, \G) = [\k_{\gntk}(G, G_1),\cdots,\k_{\gntk}(G,G_n)]^\top \in \R^n.
\end{align*}
We denote the feature function corresponding to the kernel $\k_{\gntk}$ as we defined above as $\Phi:\R^{d \times N} \rightarrow \mathcal{F}$, which satisfies
\begin{align*}
	\langle\Phi(G),\Phi(H)\rangle_\mathcal{F} = \k_{\gntk}(G,H),
\end{align*}
for any graph data $G$, $H$. And we write $\Phi(\G)=[\Phi(G_1),\cdots,\Phi(G_n)]^\top$.
\end{definition}

\begin{definition}[Neural tangent kernel regression]\label{def:krr_ntk}
Given
\begin{itemize}
    \item training data matrix $\G = \{G_1, \cdots , G_n\}$
    \item  and corresponding label vector $Y\in\R^n$.
\end{itemize}
 Let
 \begin{itemize}
     \item  $\k_{\gntk}$, $\H^{\cts}$ and $\Phi$ be the neural tangent kernel and corresponding feature functions defined as in Definition~\ref{def:ntk_phi}
 \end{itemize}
Then we consider the following neural tangent kernel regression problem:
\begin{align}\label{eq:krr}
\min_{\beta} \frac{1}{2}\| Y - \kappa f_{\gntk}(\beta,\G) \|_2^2 
\end{align}
where $f_{\gntk}(\beta,G) = \Phi(G)^\top \beta \in \R$ denotes the prediction function 
and 
\begin{align*}
f_{\gntk}(\beta,\G) = [f_{\gntk}(\beta,G_1),\cdots,f_{\gntk}(\beta,G_n)]^\top\in\R^{n}.
\end{align*}

Consider the gradient flow associated with solving the problem given by equation~\eqref{eq:krr}, starting with the initial condition $\beta(0) = 0$. Let $\beta(t)$ represent the variable at the $t$-th iteration. We represent
\begin{align}\label{eq:ntk_predict_train}
	u_{\gntk}(t) = \kappa\Phi(\G)\beta(t) \in \R^n
\end{align} 
as the training data predictor at iteration $t$. Given any test data $x_{\test}\in\R^d$, we denote
\begin{align}\label{eq:ntk_predict_test}
	u_{\gntk,\test}(t) = \kappa\Phi(G_{\test})^\top\beta(t) \in \R
\end{align} 
as the predictor for the test data at the $t$-th iteration. It's important to highlight that the gradient flow converges to the optimal solution of problem~\eqref{eq:krr} because of the strong convexity inherent to the problem
. We denote
\begin{align}\label{eq:def_beta_*}
	\beta^* = \lim_{t\to\infty} \beta(t) = \kappa^{-1} ( \Phi(\G)^\top\Phi(\G) )^{-1} \Phi(\G)^\top Y
\end{align}
and the optimal predictor derived from training data
\begin{align}\label{eq:def_u_*}
	u^* = \lim_{t \to \infty} u_{\gntk}(t) = \kappa \Phi(G)\beta^* = 
	Y \in \R^n
\end{align}
and the optimal test data predictor
\begin{align}\label{eq:def_u_test_*}
	u_{\test}^* = & ~ \lim_{t\to\infty} u_{\gntk,\test}(t) \notag \\
	= & ~ \kappa \Phi(G_{\test})^\top \beta^* \notag \\
	= & ~ 
	\k_{\gntk}(G_{\test}, \G)^\top(
	\H^{\cts} )^{-1}Y \in \R.
\end{align}
\end{definition}

\begin{definition}[Dynamic kernel]\label{def:dynamic_kernel}
Given 
\begin{itemize}
    \item $W(t) \in \R^{d \times m}$ as the parameters of the neural network at training time $t$ as defined in Definition~\ref{def:nn}.
\end{itemize}
 For any data $G,H\in\R^{d \times N}$
, we define $\k_t(G,H)\in\R$ as
\begin{align*}
    \k_t(G,H)
    = \left\langle \frac{\partial f_{\gnn}(W(t),G)}{\partial W(t)},\frac{\partial f_{\gnn}(W(t),H)}{\partial W(t)} \right\rangle
\end{align*}
Given training data matrix $\G=\{ G_1, \cdots, G_n \}$, we define $\H(t) \in\R^{n\times n}$ as
\begin{align*}
	[\H(t)]_{i,j} = \k_{t}(G_i, G_j)\in\R.
\end{align*}
Further, given a test graph data $G_{\test}$, we define $\k_t(G_{\test},\G)\in\R^n$ as
\begin{align*}
    \k_t(G_{\test},\G) = [\k_t(G_{\test},G_1), \cdots, \k_t(G_{\test},G_n)]^\top\in\R^n.
\end{align*}
\end{definition}

\subsection{Gradient flow}\label{sec:equiv_gradient_flow} 
In this section, we compute the gradient flow of iterative GNTK regression and GNN train process.

First, we compute the gradient flow of kernel regression.
\begin{lemma}\label{lem:gradient_flow_of_krr}
Given
\begin{itemize}
    \item training graph data $\G = \{ G_1, \cdots, G_n \}$,
    \item  and corresponding label vector $Y \in \R^n$. Let $f_{\gntk}$ be defined as in Definition~\ref{def:krr_ntk}.
\end{itemize}
 Let 
 \begin{itemize}
     \item  $\beta(t)$, $\kappa\in(0,1)$ and $u_{\gntk}(t)\in\R^n$ be defined as in Definition~\ref{def:krr_ntk}.
     \item $\k_{\gntk}: \R^d \times \R^{n\times d} \to\R^n$ be defined as in Definition~\ref{def:ntk_phi}.
 \end{itemize}
  Then for any data $G\in\R^{d \times N}$, we have
\begin{align*}
	\frac{\d f_{\gntk}(\beta(t), G)}{\d t} = \kappa \cdot \k_{\gntk}(G, \G )^\top ( Y - u_{\gntk}(t) ) 
\end{align*}
\end{lemma}
\begin{proof}
Denote $L(t)= \frac{1}{2}\|Y-u_{\gntk}(t)\|_2^2$. 
By the rule of gradient descent, we have
\begin{align*}
	\frac{\d \beta(t)}{\d t}=-\frac{\d L}{\d \beta}=\kappa \Phi( \G )^\top(Y-u_{\gntk}(t)), 
\end{align*}
where $\Phi$ is defined in Definition~\ref{def:ntk_phi}.
Thus we have
\begin{align*}
	\frac{\d f_{\gntk}(\beta(t), G)}{\d t}
	= & ~ \frac{\d f_{\gntk}(\beta(t), G)}{\d \beta(t)}\frac{\d \beta(t)}{\d t} \\
	= & ~ \Phi(G)^\top (\kappa\Phi(\G)^\top(Y-u_{\gntk}(t)) ) \\
	= & ~ \kappa\k_{\gntk}(G, \G)^\top (Y-u_{\gntk}(t)) 
\end{align*}
where the initial step arises from the chain rule. The subsequent step is a consequence of the relationship $ \d f_{\gntk}(\beta, G)/ \d \beta=\Phi(G)^\top$. The final step is based on the definition of the kernel, with $\k_{\gntk}(G, \G) = \Phi(\G) \Phi( G )$ belonging to $\R^{n}$.
\end{proof}

\begin{corollary}[Gradient of prediction of kernel regression]\label{cor:ntk_gradient}
Given
\begin{itemize}
    \item training data matrix $\G = \{ G_1, \cdots, G_n \}$ and corresponding label vector $Y \in \R^n$.
    \item a test data $G_{\test}$.
\end{itemize}
 Let 
 \begin{itemize}
     \item $f_{\gntk}$ be defined as in Definition~\ref{def:krr_ntk}.
     \item $\beta(t)$, $\kappa\in(0,1)$ and $u_{\gntk}(t)\in\R^n$ be defined as in Definition~\ref{def:krr_ntk}.
     \item   $\k_{\gntk}: \R^d \times \R^{n \times d} \rightarrow \R^n,~H^{\cts} \in \R^{n\times n}$ be defined as in Definition~\ref{def:ntk_phi}. 
 \end{itemize}
Then we have 
\begin{align*}
	\frac{\d u_{\gntk}(t)}{\d t} & = \kappa^2 H^{\cts} ( Y - u_{\gntk}(t) ) \\
	\frac{\d u_{\gntk, \test}(t)}{\d t} & = \kappa^2 \k_{\gntk}( G_{\test}, X)^\top  ( Y - u_{\gntk}(t) ). 
\end{align*}
\end{corollary}
\begin{proof}
Plugging in $G = G_i $ in Lemma~\ref{lem:gradient_flow_of_krr}, we have
\begin{align*}
	\frac{\d f_{\gntk}(\beta(t), G_i)}{\d t} = \kappa \k_{\gntk}(G_i, \G)^\top ( Y - u_{\gntk}(t) ) . 
\end{align*}
Note $[u_{\gntk}(t)]_i = \kappa f_{\gntk}(\beta(t), x_i)$ and $[ \H^{\cts} ]_{:,i} = \k_{\gntk}(x_i, X)$, so writing all the data in a compact form, we have
\begin{align*}
	\frac{\d u_{\gntk}(t)}{\d t} = \kappa^2 \H^{\cts} ( Y - u_{\gntk}(t) ) . 
\end{align*}
Plugging in data $G = G_{\test}$ in Lemma~\ref{lem:gradient_flow_of_krr}, we have
\begin{align*}
	\frac{\d f_{\gntk}(\beta(t), G_{\test})}{\d t} = \kappa \k_{\gntk}( G_{\test}, \G)^\top ( Y - u_{\gntk}(t) ) . 
\end{align*}
Note by definition, $u_{\gntk,\test}(t) = \kappa f_{\gntk}(\beta(t), G_{\test}) \in \R$, so we have
\begin{align*}
	\frac{\d u_{\gntk, \test}(t)}{\d t} = \kappa^2 \k_{\gntk}( G_{\test}, \G )^\top ( Y - u_{\gntk}(t) ) . 
\end{align*}
\end{proof}

We prove the linear convergence of kernel regression:
\begin{lemma}\label{lem:linear_converge_krr}
Given
\begin{itemize}
    \item training graph data $\G = \{ G_1, \cdots, G_n \}$,
    \item  and corresponding label vector $Y\in\R^n$.
\end{itemize}
Let
\begin{itemize}
    \item $\kappa \in (0,1)$ and $u_{\gntk}(t) \in \R^n$ be defined as in Definition~\ref{def:krr_ntk}.
    \item  $u^* \in \R^n$ be defined in Definition~\ref{def:krr_ntk}.
    \item $\Lambda_0 > 0$ be defined as in Definition~\ref{def:ntk_phi}.
\end{itemize}
 Then we have
\begin{align*}
\frac{\d \|u_{\gntk}(t)-u^*\|_2^2}{\d t} \le - 2(\kappa^2 \Lambda_0) \|u_{\gntk}(t)-u^*\|_2^2.
\end{align*}
Further, we have
\begin{align*}
	\|u_{\gntk}(t)-u^*\|_2 \leq e^{-(\kappa^2 \Lambda_0 )t} \|u_{\gntk}(0)-u^*\|_2.
\end{align*}

\end{lemma}
\begin{proof}
So we have
\begin{align}\label{eq:322_2}
	& ~ \frac{\d \|u_{\gntk}(t)-u^*\|_2^2}{\d t} \notag \\
	= & ~ 2(u_{\gntk}(t)-u^*)^\top \frac{\d u_{\gntk}(t)}{\d t} \notag\\
	= & ~ -2\kappa^2 (u_{\gntk}(t)-u^*)^\top H^{\cts} (u_{\gntk}(t) - Y) \notag \\ 
	\leq & ~ -2(\kappa^2 \Lambda_0 )\|u_{\gntk}(t)-u^*\|_2^2,
\end{align}
where the initial step is derived from the chain rule. The subsequent step is based on Corollary~\ref{cor:ntk_gradient}, and the concluding step is in accordance with the definition of $\Lambda_0$.
Further, since 
\begin{align*}
	& ~ \frac{\d (e^{2(\kappa^2 \Lambda_0 )t}\|u_{\gntk}(t)-u^*\|_2^2)}{\d t} \\
	= & ~ 2(\kappa^2 \Lambda_0 )e^{2(\kappa^2 \Lambda_0 )t}\|u_{\gntk}(t)-u^*\|_2^2 \\
	&~\quad+ e^{2(\kappa^2 \Lambda_0 )t}\cdot\frac{\d \|u_{\gntk}(t)-u^*\|_2^2}{\d t} \\
	\leq & ~ 0,
\end{align*}
where the initial step involves gradient computation, while the subsequent step is derived from Eq.~\eqref{eq:322_2}. Consequently, the value of $e^{2(\kappa^2 \Lambda_0 )t}|u_{\gntk}(t)-u^*|_2^2$ does not increase, leading to the inference that
\begin{align*}
	\|u_{\gntk}(t)-u^*\|_2 \leq e^{-2(\kappa^2 \Lambda_0)t} \|u_{\gntk}(0)-u^*\|_2.
\end{align*}
\end{proof}

We compute the gradient flow of neural network training:
\begin{lemma}\label{lem:gradient_flow_of_nn}
Given
\begin{itemize}
    \item training graph data $\G = \{ G_1, \cdots, G_n \}$ ,
    \item and corresponding label vector $Y\in\R^n$.
\end{itemize}
 Let
 \begin{itemize}
     \item  $f_{\gnn}: \R^{d\times m} \times \R^{d \times N} \rightarrow \R$ be defined as in Definition~\ref{def:f_nn}.
\item $W(t) \in \R^{d \times m}$, $\kappa\in(0,1)$ and $u_{\gnn}(t)\in\R^n$ be defined as in Definition~\ref{def:nn}.
\item  $\k_{t}: \R^{d \times N} \times \R^{n\times d \times N} \rightarrow \R^n$ be defined as in Definition~\ref{def:dynamic_kernel}.
 \end{itemize}
 Then for any data $G \in \R^{d\times N}$, we have
\begin{align*}
\frac{\d f_{\gnn}(W(t),G)}{\d t} = \kappa \k_{t}( G , \G )^\top ( Y - u_{\gnn}(t) ) . 
\end{align*}
\end{lemma}
\begin{proof}
Denote $L(t)=\frac{1}{2}\|Y-u_{\gnn}(t)\|_2^2$. 

By the rule of gradient descent, we have
\begin{align}\label{eq:323_1}
	\frac{\d w_r}{\d t} = -\frac{\partial L}{\partial w_r}=(\frac{\partial u_{\gnn}}{\partial w_r})^\top(Y-u_{\gnn}) . 
\end{align} 
Thus, we have
\begin{align*}
 & ~ \frac{\d f_{\gnn}(W(t),G)}{\d t} \\
= & ~ \Big\langle \frac{\d f_{\gnn}(W(t),G)}{\d W(t)}, \frac{\d W(t)}{\d t} \Big\rangle \notag \\
= & ~ \sum_{j=1}^{n}(y_j - \kappa f_{\gnn}(W(t),G_j)) \\
&\quad \cdot \Big\langle \frac{\d f_{\gnn}(W(t),G)}{\d W(t)},\frac{\d \kappa f_{\gnn}(W(t),G_j)}{\d W(t)} \Big\rangle \notag \\ 
= & ~ \kappa \sum_{j=1}^{n}(y_j- \kappa f_{\gnn}(W(t),G_j)) \cdot \k_{t}(G,G_j)\notag \\ 
= & ~ \kappa \k_{t}(G , \G )^\top ( Y - u_{\gnn}(t) ) 
\end{align*}
where the initial step is derived from the chain rule. The subsequent step is based on Eq.~\eqref{eq:323_1}. The third step adheres to the definition of $\k_{t}$, and the concluding step presents the formula in a more concise manner.
\end{proof}

\begin{corollary}[Gradient of prediction of graph neural network]\label{cor:nn_gradient}
Given 
\begin{itemize}
    \item training graph data $\G = \{ G_1, \cdots, G_n \}$ and corresponding label vector $Y \in \R^n$.
    \item a test graph data $G_{\test}$. Let $f_{\gnn}:\R^{d\times m} \times \R^{d \times N} \rightarrow \R$ be defined as in Definition~\ref{def:f_nn}.
\end{itemize}
Let
\begin{itemize}
    \item $W(t) \in \R^{d \times m}$, $\kappa\in(0,1)$ and $u_{\gnn}(t) \in \R^n$ be defined as in Definition~\ref{def:nn}.
    \item  $\k_{t} : \R^{d\times N} \times \R^{n \times d \times N} \rightarrow \R^n,~H(t) \in \R^{n \times n}$ be defined as in Definition~\ref{def:dynamic_kernel}. 
\end{itemize}
 Then we have 
\begin{align*}
	\frac{\d u_{\gnn}(t)}{\d t} = & ~ \kappa^2 H(t) ( Y - u_{\gnn}(t) ) \\ 
	\frac{\d u_{\gnn,\test}(t)}{\d t} = & ~ \kappa^2 \k_{t}(G_{\test}, \G)^\top ( Y - u_{\gnn}(t) ) .  
\end{align*}
\end{corollary}
\begin{proof}
Plugging in $G = G_i\in\R^d$ in Lemma~\ref{lem:gradient_flow_of_nn}, we have
\begin{align*}
	\frac{\d f_{\gnn}(W(t), G_i)}{\d t} = \kappa \k_{t}(G_i, \G)^\top ( Y - u_{\gnn}(t) ) . 
\end{align*}
Note $[u_{\gnn}(t)]_i = \kappa f_{\gnn}(W(t), G_i)$ and $[H(t))]_{:,i} = \k_{t}(G_i, \G)$, so writing all the data in a compact form, we have
\begin{align*}
	\frac{\d u_{\gnn}(t)}{\d t} = \kappa^2 H(t) ( Y - u_{\gnn}(t) ). 
\end{align*}
Plugging in data $G = G_{\test} $ in Lemma~\ref{lem:gradient_flow_of_nn}, we have
\begin{align*}
	\frac{\d f_{\gnn}(W(t), G_{\test})}{\d t} = \kappa \k_{t}( G_{\test}, \G)^\top ( Y - u_{\gnn}(t) ). 
\end{align*}
Note by definition, $u_{\gnn,\test}(t) = \kappa f_{\gnn}(W(t), G_{\test}) $, so we have
\begin{align*}
	\frac{\d u_{\gnn, \test}(t)}{\d t} = \kappa^2 \k_{t}( G_{\test}, \G)^\top ( Y - u_{\gnn}(t) ). 
\end{align*}
\end{proof}

We prove the linear convergence of neural network training:
\begin{lemma}\label{lem:linear_converge_nn}
Given
\begin{itemize}
    \item training graph data matrix $\G = \{ G_1, \cdots, G_n \}$ and corresponding label vector $Y \in \R^n$.
    \item  the total number of iterations $T>0$.
\end{itemize}
  Let 
  \begin{itemize}
      \item $\kappa\in(0,1)$ and $u_{\gnn}(t) \in \R^{n \times n}$ be defined as in Definition~\ref{def:nn}. 
      \item  $u^* \in \R^n$ be defined in Eq.~\eqref{eq:def_u_*}.
      \item $H^{\cts} \in \R^{n \times n}$ and $\Lambda_0 > 0$ be defined as in Definition~\ref{def:ntk_phi}.
      \item $H(t) \in \R^{n \times n}$ be defined as in Definition~\ref{def:dynamic_kernel}.
  \end{itemize}
Then we have
\begin{align*}
	\frac{\d \|u_{\gnn}(t)-u^*\|_2^2}{\d t} \le - ( \kappa^2 \Lambda_0) \|u_{\gnn}(t)-u^*\|_2^2.
\end{align*}

\end{lemma}

\begin{proof}

Thus, we have
\begin{align*}
 & ~ \frac{\d \|u_{\gnn}(t)-u^*\|_2^2}{\d t} \\
= & ~  2(u_{\gnn}(t)-u^*)^\top \frac{\d u_{\gnn}(t)}{\d t}\\
= & ~ -2 \kappa^2 (u_{\gnn}(t)-u^*)^\top H(t) (u_{\gnn}(t) - Y) \\
= & ~ -2(u_{\gnn}(t)-u^*)^\top ( \kappa^2 H(t) ) (u_{\gnn}(t) - u^*) 
\\
\leq & ~ -2 ( \kappa^2 \Lambda_0  ) \| u_{\gnn}(t) - u^* \|_2^2 
\end{align*}
where the first step follows the chain rule, the second step follows Corollary~\ref{cor:nn_gradient}, the third step uses basic linear algebra, 
and the last step follows the assumption $\|H(t) - H^{\cts}\| \leq \Lambda_0/2$.
\end{proof}


\subsection{Perturbation during iterative GNTK regression}\label{sec:equiv_bound_ntk_test_T_and_test_*} 
In this section, we prove an upper bound for $| u_{\gntk,\test}(T) - u_{\test}^* |$. 

\begin{lemma}\label{lem:u_ntk_test_T_minus_u_test_*}
Let 
\begin{itemize}
    \item $u_{\gntk,\test}(T) \in \R$ and $u_{\test}^* \in \R$ be defined as Definition~\ref{def:krr_ntk}.
\end{itemize}
Given 
\begin{itemize}
    \item any accuracy $\epsilon>0$,
\end{itemize}
 if $\kappa\in(0,1)$, then by picking $T = \wt{O}(\frac{1}{\kappa^2 \Lambda_0})$, we have 
\begin{align*}
| u_{\gntk,\test}(T) - u_{\test}^* | \leq \epsilon/2.
\end{align*}
\end{lemma}
where $\wt{O}(\cdot)$ here hides $\poly\log( n/(\epsilon \Lambda_0) )$.

\begin{proof}
Due to the linear convergence of kernel regression,
i.e.,
\begin{align*}
	\frac{ \d \| \beta(t) - \beta^* \|_2^2 }{ \d t } \leq - 2 ( \kappa^2 \Lambda_0 ) \| \beta(t) - \beta^* \|_2^2
\end{align*}
Thus,
\begin{align*}
	& ~ | u_{\gntk,\test}(T) - u_{\test}^* | \\
	= & ~ | \kappa \Phi( G_{\test} )^\top \beta(T) - \kappa \Phi( G_{\test} )^\top\beta^* | \\
	\leq & ~ \kappa \| \Phi( G_{\test} ) \|_2 \| \beta(T) - \beta^* \|_2\\
	\leq & ~ \kappa e^{-(\kappa^2 \Lambda_0 )T}\| \beta(0) - \beta^* \|_2 \\
	\leq & ~ e^{-(\kappa^2 \Lambda_0 )T} \cdot \poly(\kappa,N,n,1/\Lambda_0)
\end{align*}
where the final step is derived from the conditions $\beta(0) = 0$ and the norm $|\beta^*|_2$ being a polynomial function of $\kappa, N, n,$ and $1/\Lambda_0$.

It's worth noting that $\kappa$ lies in the interval (0,1). Consequently, by selecting $ T = \wt{O}(\frac{1}{\kappa^2\Lambda_0})$, it follows that
\begin{align*}
	\| u_{\gntk,\test}(T) - u_{\test}^* \|_2 \leq \epsilon/2,
\end{align*}
where $\wt{O}(\cdot)$ here hides $\poly\log( N n/ (\epsilon \Lambda_0) )$.
\end{proof}

\begin{lemma}\label{lem:u_ntk_test_T_minus_u_test_*_node}
Let $u_{\gntk,\test,\node}(T) \in \R$ and $u_{\test,\node}^* \in \R$ be defined as Definition~\ref{def_main:krr_ntk_node}. 
Given
\begin{itemize}
    \item any accuracy $\epsilon>0$,
\end{itemize}
 if 
 \begin{itemize}
     \item  $\kappa\in(0,1)$,
 \end{itemize}
 then by picking $T = \wt{O}(\frac{1}{\kappa^2 \Lambda_0})$, we have 
\begin{align*}
| u_{\gntk,\test,\node}(T) - u_{\test,\node}^* | \leq \epsilon/2.
\end{align*}
\end{lemma}

\subsection{Bounding initialization and kernel perturbation}\label{sec:equiv_splitting_nn_test_T_and_ntk_test_T_into_three} 
In this section we prove Lemma~\ref{lem:more_concreate_bound}, which shows that in order to bound prediction perturbation, it suffices to bound both initialization perturbation and kernel perturbation. We prove this result by utilizing Newton-Leibniz formula.

\begin{lemma}[Prediction perturbation implies kernel perturbation]\label{lem:more_concreate_bound}
Given
\begin{itemize}
    \item training graph data $\G = \{G_1, \cdots, G_n \}$ and corresponding label vector $Y \in \R^n$.
    \item  the total number of iterations $T > 0$.
    \item arbitrary test data $G_{\test}$.
\end{itemize}
   Let
   \begin{itemize}
       \item $u_{\gnn,\test}(t) \in \R^n$ and $u_{\gntk,\test}(t) \in \R^n$ be the test data predictors defined in Definition~\ref{def:nn} and Definition~\ref{def:krr_ntk} respectively.
       \item $\kappa\in(0,1)$ be the corresponding multiplier.
       \item $\k_{\gntk}( G_{\test}, \G ) \in \R^n,~\k_{t}(G_{\test},\G) \in \R^n,~\H^{\cts} \in \R^{n \times n},~\H(t) \in \R^{n \times n},~\Lambda_0 > 0$ be defined in Definition~\ref{def:ntk_phi} and Definition~\ref{def:dynamic_kernel}.
       \item $u^* \in \R^n$ be defined as in Eq.~\eqref{eq:def_u_*}. 
       \item $\epsilon_{K} \in (0,1)$, $\epsilon_{\init} \in (0,1)$ and $\epsilon_H \in (0,1)$ denote parameters that are independent of $t$,  and the following conditions hold for all $t\in[0,T]$,
\begin{itemize}
	\item $\|u_{\gnn}(0)\|_2 \leq \sqrt{n}\epsilon_{\init}$ and $|u_{\gnn,\test}(0)| \leq \epsilon_{\init}$
    \item $\|\k_{\gntk}( G_{\test}, \G ) - \k_{t} ( G_{\test},\G ) \|_2\le \epsilon_{K}$
    \item $\| \H(t) - \H^{\cts} \| \le \epsilon_H$
\end{itemize}
   \end{itemize}

then we have
\begin{align*}
    &~|u_{\gnn,\test}(T)-u_{\gntk,\test}(T)| \\
    \leq & ~ (1+\kappa^2 nT)\epsilon_{\init} +  \epsilon_K \cdot  \frac{ \| u^* \|_2 }{  \Lambda_0  }  + \sqrt{n}T^2\kappa^4 \epsilon_H \| u^* \|_2
\end{align*}
\end{lemma}
\begin{proof}
Combining results from Lemma~\ref{lem:very_rough_bound}, Claim~\ref{cla:A}.~\ref{cla:B},~\ref{cla:C}, we complete the proof.
We have
\begin{align*}
& ~ |u_{\gnn,\test}(T)-u_{\gntk,\test}(T)| \\
\leq & ~ |u_{\gnn,\test}(0)-u_{\gntk,\test}(0)|\\
& ~ + \kappa^2 \Big| \int_{0}^T (\k_{\gntk}( G_{\test}, \G )-\k_{t}( G_{\test}, \G ))^\top (u_{\gntk}(t)-Y) \d t \Big|\\
& ~ + \kappa^2 \Big| \int_{0}^T \k_{t}( G_{\test}, \G )^\top(u_{\gntk}(t)-u_{\gnn}(t)) \d t \Big|\\
\leq & ~ \epsilon_{\init} + \epsilon_K \cdot  \frac{ \|u^*\| }{ \Lambda_0  } + \kappa^2 n\epsilon_{\init}T + \sqrt{n}T^2 \cdot  \kappa^4 \epsilon_H \cdot  \| u^* \|_2 \\
\leq & ~ (1+ \kappa^2 nT)\epsilon_{\init} + \epsilon_K \cdot \frac{ \| u^* \|_2 }{\Lambda_0 } + \sqrt{n}T^2\kappa^4 \epsilon_H  \| u^* \|_2   
\end{align*}
where the initial step is derived from Lemma~\ref{lem:very_rough_bound}. The subsequent step is based on Claim~\ref{cla:A},~\ref{cla:B}, and \ref{cla:C}. The concluding step streamlines the expression.
\end{proof}

To prove Lemma~\ref{lem:more_concreate_bound}, we first show that $| u_{\gnn,\test}(T) - u_{\gntk,\test}(T) |$ can be upper bounded by three terms that are defined in Lemma~\ref{lem:very_rough_bound}, then we bound each of these three terms respectively in Claim~\ref{cla:A}, Claim~\ref{cla:B}, and Claim~\ref{cla:C}.

\begin{lemma}\label{lem:very_rough_bound}
Follow the same notation as~Lemma~\ref{lem:more_concreate_bound}, we have
\begin{align*}
&~| u_{\gnn,\test}(T) - u_{\gntk,\test}(T) | \\
\leq & ~ |u_{\gnn,\test}(0)-u_{\gntk,\test}(0)| \\
    &\quad + \Big| \int_{0}^T \kappa^2 (\k_{\gntk}( G_{\test}, \G )-\k_{t}(  G_{\test} , \G  ))^\top (u_{\gntk}(t)-Y) \d t \Big|\\
    & \quad + \Big| \int_{0}^T  \kappa^2 \k_{t}( G_{\test}, \G)^\top(u_{\gntk}(t)-u_{\gnn}(t)) \d t \Big|.
\end{align*}
\end{lemma}
\begin{proof}

\begin{align}\label{eq:320_3}
    & ~|u_{\gnn,\test}(T)-u_{\gntk,\test}(T)|\notag\\
    = & ~ \Big| u_{\gnn,\test}(0)-u_{\gntk,\test}(0)\notag \\
    &\quad +\int_{0}^T(\frac{\d u_{\gnn,\test}(t)}{\d t}-\frac{\d u_{\gntk,\test}(t)}{\d t})\d t \Big |\notag\\
    \leq & ~ | u_{\gnn,\test}(0)-u_{\gntk,\test}(0) |\notag \\
    &\quad+ \Big| \int_{0}^T(\frac{\d u_{\gnn,\test}(t)}{\d t}-\frac{\d u_{\gntk,\test}(t)}{\d t}) \d t \Big|,
\end{align}
where the initial step is derived from the integral's definition. The subsequent step is based on the triangle inequality. It's worth noting that, as per Corollary~\ref{cor:ntk_gradient} and ~\ref{cor:nn_gradient}, their respective gradient flows are described as
\begin{align}
    \frac{\d u_{\gntk,\test}(t)}{\d t}&= - \kappa^2 \k_{\gntk}( G_{\test} , \G )^\top(u_{\gntk}(t)-Y)\label{eq:320_1}\\ 
    \frac{\d u_{\gnn,\test}(t)}{\d t}&= - \kappa^2 \k_{t}( G_{\test} , \G )^\top(u_{\gnn}(t)-Y)\label{eq:320_2} 
\end{align}
where $u_{\gntk}(t) \in \R^n$ and $u_{\gnn}(t) \in \R^n$ are the predictors for training data defined in Definition~\ref{def:krr_ntk} and Definition~\ref{def:nn}. Thus, we have 
\begin{align}\label{eq:320_4}
    & ~ \frac{\d u_{\gnn,\test}(t)}{\d t}-\frac{\d u_{\gntk,\test}(t)}{\d t} \notag\\
    = & ~ - \kappa^2 \k_{t}( G_{\test} , \G )^\top(u_{\gnn}(t)-Y)\notag\\
    &\quad + \kappa^2 \k_{\gntk}( G_{\test} , \G )^\top(u_{\gntk}(t)-Y) \notag \\ 
    = & ~  \kappa^2 (\k_{\gntk}( G_{\test} , \G )- \k_{t}( G_{\test} , \G ))^\top (u_{\gntk}(t)-Y)\notag \\
    &\quad-  \kappa^2 \k_{t}( G_{\test} , \G )^\top(u_{\gntk}(t)-u_{\gnn}(t)) 
\end{align}
where the initial step is derived from both Eq.\eqref{eq:320_1} and Eq.\eqref{eq:320_2}. The subsequent step reformulates the formula.
So we have
\begin{align}\label{eq:320_5}
    & ~ \Big|\int_{0}^T (\frac{\d u_{\gnn,\test}(t)}{\d t}-\frac{\d u_{\gntk,\test}(t)}{\d t}) \d t\Big| \notag\\
    = & ~ \Big|\int_{0}^T  \kappa^2 ((\k_{\gntk}(x_{\test},X)-\k_{t}( G_{\test}, \G ))^\top (u_{\gntk}(t)-Y)\notag \\
    &\quad- \kappa^2  \k_{t}( G_{\test}, \G )^\top(u_{\gntk}(t)-u_{\gnn}(t))) \d t\Big|
\end{align}
Thus,
\begin{align*}
    & ~ |u_{\gnn,\test}(T)-u_{\gntk,\test}(T)|\\
    \leq & ~ |u_{\gnn,\test}(0)-u_{\gntk,\test}(0)| \\
    &\quad + \Big| \int_{0}^T(\frac{\d u_{\gnn,\test}(t)}{\d t}-\frac{\d u_{\gntk,\test}(t)}{\d t}) \d t \Big|\\
    \leq & ~ |u_{\gnn,\test}(0)-u_{\gntk,\test}(0)| \\
    &\quad + \Big|\int_{0}^T  \kappa^2 ((\k_{\gntk}( G_{\test},\G )-\k_{t}( G_{\test}, \G ))^\top (u_{\gntk}(t)-Y) \\
    & \quad -  \kappa^2 \k_{t}( G_{\test},\G )^\top(u_{\gntk}(t)-u_{\gnn}(t))) \d t\Big|\\
    \leq & ~ |u_{\gnn,\test}(0)-u_{\gntk,\test}(0)| \\
    &\quad + \Big| \int_{0}^T \kappa^2 (\k_{\gntk}( G_{\test}, \G )-\k_{t}(  G_{\test} , \G  ))^\top (u_{\gntk}(t)-Y) \d t \Big|\\
    & \quad + \Big| \int_{0}^T  \kappa^2 \k_{t}( G_{\test}, \G)^\top(u_{\gntk}(t)-u_{\gnn}(t)) \d t \Big|
\end{align*}
where the initial step is based on Eq.\eqref{eq:320_3}. The next step is derived from Eq.\eqref{eq:320_5}, and the third step arises from the triangle inequality.
\end{proof}

Now we bound each of these three terms $|u_{\gnn,\test}(0)-u_{\gntk,\test}(0)|$, $\Big| \int_{0}^T \kappa^2 (\k_{\gntk}( G_{\test}, \G )-\k_{t}(  G_{\test} , \G  ))^\top (u_{\gntk}(t)-Y) \d t \Big|$ and $\Big| \int_{0}^T  \kappa^2 \k_{t}( G_{\test}, \G)^\top(u_{\gntk}(t)-u_{\gnn}(t)) \d t \Big|$ in the following three claims.

\begin{claim}\label{cla:A}\
We have
\begin{align*}
	|u_{\gnn,\test}(0)-u_{\gntk,\test}(0)| \leq \epsilon_{\init}.
\end{align*}
\end{claim}
\begin{proof}
	Note $u_{\gntk,\test}(0) = 0$, so by assumption we have 
\begin{align*}
	|u_{\gnn,\test}(0)-u_{\gntk,\test}(0)|=|u_{\gnn,\test}(0)| \leq \epsilon_{\init}.
\end{align*}

\end{proof}

\begin{claim}\label{cla:B}
We have
\begin{align*}
\Big| \int_{0}^T \kappa^2 (\k_{\gntk}( G_{\test}, \G )-\k_{t}(  G_{\test} , \G  ))^\top (u_{\gntk}(t)-Y) \d t \Big| \leq \epsilon_K \cdot \frac{ \|u^*\| }{ \Lambda_0 }.
\end{align*}
\end{claim}

\begin{proof}
Note
\begin{align*}
 & ~ \kappa^2 \Big| \int_{0}^T (\k_{\gntk}( G_{\test} , \G )-\k_{t} ( G_{\test} , \G ))^\top (u_{\gntk}(t)-Y) \d t \Big|\\
\le & ~ \kappa^2 \max_{ t \in [ 0 , T ] }\|\k_{\gntk}( G_{\test} , \G )-\k_{t}( G_{\test} , \G )\|_2 \\
&\quad \cdot \int_{0}^T \| u_{\gntk}(t) - Y \|_2 \d t ,
\end{align*}
where the initial step is derived from the Cauchy-Schwarz inequality. It's important to mention that, as per Lemma~\ref{lem:linear_converge_krr}, the kernel regression predictor $u_{\gntk}(t)$, which belongs to $\R^n$, exhibits linear convergence towards the optimal predictor $u^*=Y$, which is also in $\R^n$. Specifically, 
\begin{align}\label{eq:320_6}
	\|u_{\gntk}(t) - u^*\|_2 \leq e^{-(\kappa^2 \Lambda_0 )t} \|u_{\gntk}(0) - u^*\|_2.
\end{align}
Thus, we have
\begin{align}\label{eq:upper_bound_int_0_T_u_ntk_t_minus_Y}
    & ~ \int_{0}^T \| u_{\gntk}(t) - Y \|_2 \d t \notag \\
    \leq & ~ \int_0^T \| u_{\gntk}(t) - u^* \|_2 \d t + \int_0^T \|u^*-Y\|_2 \d t \notag \\
    \le & ~ \int_{0}^T e^{-(\kappa^2 \Lambda_0 )} \|u_{\gntk}(0)-u^*\|_2 \d t  \notag \\
    \leq & ~  \frac{\|u_{\gntk}(0)-u^*\|_2}{\kappa^2 \Lambda_0 }  \notag \\
    = & ~   \frac{ \|u^*\| }{\kappa^2 \Lambda_0 } ,
\end{align}
where the initial step arises from the triangle inequality. The subsequent step is based on Eq.~\eqref{eq:320_6}. The third step involves computing the integration, and the concluding step is derived from the condition $u_{\gntk}(0) = 0$.
Thus, we have
\begin{align*}
    & ~ \Big| \int_{0}^T \kappa^2 (\k_{\gntk}( G_{\test}, \G )-\k_{t}(  G_{\test} , \G  ))^\top (u_{\gntk}(t)-Y) \d t \Big| \\
    \le & ~ \kappa^2 \max_{t \in [0, T]}\|\k_{\gntk}( G_{\test}, \G)-\k_{t}( G_{\test}, \G )\|_2   \cdot \int_0^T \| u_{\gntk}(t) - Y \|_2 \d t \\
    \leq &~ \epsilon_K \cdot \frac{ \|u^*\| }{ \Lambda_0 }.
\end{align*}
where the initial step is derived from Eq.\eqref{eq:320_6}. The subsequent step is based on both Eq.\eqref{eq:upper_bound_int_0_T_u_ntk_t_minus_Y} and the definition of $\epsilon_K$.
\end{proof}

\begin{claim}\label{cla:C}
We have
\begin{align*}
\Big| \int_{0}^T  \kappa^2 \k_{t}( G_{\test}, \G)^\top(u_{\gntk}(t)-u_{\gnn}(t)) \d t \Big| \leq n\epsilon_{\init}T + \sqrt{n}T^2 \cdot  \kappa^2 \epsilon_H \cdot \| u^* \|_2 
\end{align*}
\end{claim}
\begin{proof}
Note 
\begin{align}\label{eq:320_10}
	& ~ \kappa^2 \Big| \int_{0}^T \k_{t}( G_{\test} , \G )^\top(u_{\gntk}(t)-u_{\gnn}(t)) \d t \Big| \notag\\
    \le & ~ \kappa^2 \max_{ t \in [0,T] }\|\k_t( G_{\test}, \G )\|_2 \max_{ t \in [0,T] } \|u_{\gntk}(t)-u_{\gnn}(t)\|_2\cdot T
\end{align}
where the first step follows the Cauchy-Schwartz inequality.

To bound term $\max_{ t \in [0,T] } \|u_{\gntk}(t)-u_{\gnn}(t)\|_2$, notice that for any $t\in[0,T]$, we have
\begin{align}\label{eq:320_7}
    & ~ \|u_{\gntk}(t)-u_{\gnn}(t)\|_2 \notag \\
    \leq & ~\|u_{\gntk}(0)-u_{\gnn}(0)\|_2 + \Big\|\int_{0}^{t}  \frac{\d (u_{\gntk}(\tau)-u_{\gnn}(\tau))}{\d \tau} \d \tau \Big\|_2 \notag\\
    = & ~ \sqrt{n} \epsilon_{\init} + \Big\|\int_{0}^{t}  \frac{\d (u_{\gntk}(\tau)-u_{\gnn}(\tau))}{\d \tau} \d \tau \Big\|_2,
\end{align}
where the first step follows the triangle inequality, and the second step follows the assumption.
Further, 
\begin{align*}
    & ~ \frac{\d (u_{\gntk}(\tau)-u_{\gnn}(\tau))}{\d \tau} \\
    = & ~ - \kappa^2 H^{\cts}(u_{\gntk}(\tau)-Y) + \kappa^2 H(\tau)(u_{\gnn}(\tau)-Y) \\ 
    = & ~ -( \kappa^2 H(\tau) )(u_{\gntk}(\tau)-u_{\gnn}(\tau)) \\
    &\quad +  \kappa^2 (H(\tau)-H^{\cts})(u_{\gntk}(\tau)-Y),
\end{align*}
where the first step follows the Corollary~\ref{cor:ntk_gradient},~\ref{cor:nn_gradient}, the second step rewrites the formula.
Since the term $-( \kappa^2 H(\tau) )(u_{\gntk}(\tau)-u_{\gnn}(\tau))$ makes $\|\int_{0}^t \frac{\d (u_{\gntk}(\tau)-u_{\gnn}(\tau))}{\d \tau} \d \tau\|_2$ smaller. 

Taking the integral and apply the $\ell_2$ norm, we have
\begin{align}\label{eq:320_8}
	& ~ \Big\| \int_{0}^t \frac{\d (u_{\gntk}(\tau)-u_{\gnn}(\tau))}{\d \tau} \d \tau \Big\|_2 \notag \\
	\leq & ~ \Big\| \int_0^t  \kappa^2 (H(\tau)-H^{\cts})(u_{\gntk}(\tau)-Y) \d \tau \Big\|_2.
\end{align}
Thus, 
\begin{align}\label{eq:320_9}
    & ~ \max_{t\in[0,T]} \|u_{\gntk}(t)-u_{\gnn}(t)\|_2 \notag \\
    \leq & ~ \sqrt{n} \epsilon_{\init} + \max_{t\in[0,T]}  \Big\| \int_{0}^t \frac{\d (u_{\gntk}(\tau)-u_{\gnn}(\tau))}{\d \tau} \d \tau \Big\|_2 \notag \\
    \leq & ~ \sqrt{n} \epsilon_{\init} + \max_{t\in[0,T]} \Big\| \int_0^t  \kappa^2 (H(\tau)-H^{\cts})(u_{\gntk}(\tau)-Y) \d \tau \Big\|_2 \notag \\
    \leq & ~ \sqrt{n} \epsilon_{\init} + \max_{t\in[0,T]} \int_0^t  \kappa^2 \| H(\tau) - H^{\cts} \| \cdot \| u_{\gntk}(\tau) - Y \|_2 \d \tau \notag \\
    \leq & ~ \sqrt{n} \epsilon_{\init} + \max_{t\in[0,T]}  \kappa^2 \epsilon_H \cdot \Big( \int_{0}^t  \| u_{\gntk}(\tau) - u^* \|_2 \d \tau +\int_{0}^t \| u^* - Y \|_2 \d \tau \Big) \notag\\
    \leq & ~ \sqrt{n} \epsilon_{\init} + \max_{t\in[0,T]}  \kappa^2 \epsilon_H \int_{0}^t  \| u_{\gntk}(0) - u^* \|_2 \d \tau  \notag\\
    \leq & ~ \sqrt{n} \epsilon_{\init} + \max_{t\in[0,T]} t \cdot  \kappa^2 \epsilon_H \cdot \|u^*\|_2  \notag \\
    \leq & ~ \sqrt{n} \epsilon_{\init} + T \cdot  \kappa^2 \epsilon_H \cdot \|u^*\|_2 
\end{align}
where the initial step is based on Eq.\eqref{eq:320_7}. The next step is derived from Eq.\eqref{eq:320_8}. The third step arises from the triangle inequality. The fourth step is informed by the condition $| H(\tau) - H^{\cts} | \leq \epsilon_H$ for all values of $\tau \leq T$, coupled with the triangle inequality. The fifth step references the linear convergence of $| u_{\gntk}(\tau) - u^* |2$, as described in Lemma~\ref{lem:linear_converge_krr}. The sixth step is based on the condition $u{\gntk}(0) = 0$. Finally, the concluding step determines the maximum value.
Therefore,
\begin{align*}
     & ~ \Big| \int_{0}^T  \kappa^2 \k_{t}( G_{\test}, \G)^\top(u_{\gntk}(t)-u_{\gnn}(t)) \d t \Big| \\
     \leq & ~ \kappa^2 \max_{ t \in [0,T] }\|\k_t( G_{\test}, \G )\|_2 \max_{ t \in [0,T] } \|u_{\gntk}(t)-u_{\gnn}(t)\|_2\cdot T \\
     \leq & ~ \kappa^2 \max_{ t \in [0,T] } \| \k_t ( G_{\test} , \G ) \|_2 \cdot (\sqrt{n} \epsilon_{\init}T + T^2 \cdot  \kappa^2 \epsilon_H \cdot  \| u^* \|_2 ) \\
     \leq & ~ \kappa^2 n\epsilon_{\init}T + \sqrt{n}T^2 \cdot  \kappa^4 \epsilon_H \cdot  \| u^* \|_2 
\end{align*}
where the initial step is derived from Eq.\eqref{eq:320_10}. The subsequent step is based on Eq.\eqref{eq:320_9}. The concluding step is informed by the condition $\k_{t}(G ,H) \leq 1$, which is consistent with the data assumptions related to graphs $G$ and $H$.
\end{proof}


\subsection{Concentration results for kernels}\label{sec:equiv_concentration_results_kernels}

For simplify, we use $x_{i,l}$ to denote $h_{G_i,u}$ where $u$ is the $l$th node of the $N$ node in $i$th graph $G_i$. For convenient, for each $i, j \in [n] \times [n]$, we write as

\begin{align*}
& ~ \Big\langle \frac{\partial f_{\gnn}(W,G_i) }{\partial W} , \frac{ \partial f_{\gnn}(W,G_j) }{ \partial W } \Big\rangle \\
= & ~ \frac{1}{m} \sum_{r=1}^m \sum_{l_1 = 1}^N \sum_{l_2 =1}^N x_{i,l_1}^\top x_{j,l_2} {\bf 1}_{ w_r^\top x_{i,l_1} \geq 0 }  {\bf 1}_{ w_r^\top x_{j,l_2} \geq 0 }
\end{align*}
Recall the classical definition of neural network
\begin{align*}
& ~ \Big\langle \frac{\partial f_{\nn}(W, x_i) }{\partial W} , \frac{ \partial f_{\nn}(W, x_j) }{ \partial W } \Big\rangle \\
= & ~ \frac{1}{m} \sum_{r=1}^m x_{i}^\top x_{j} {\bf 1}_{ w_r^\top x_{i} \geq 0 }  {\bf 1}_{ w_r^\top x_{j} \geq 0 }
\end{align*}

In this section, we present proof for a finding that is a broader variant of Lemma 3.1 from \cite{sy19}. This result demonstrates that when the width, denoted as $m$, is adequately expansive, the continuous and discrete renditions of the gram matrix for input data exhibit proximity in a spectral manner.
\begin{lemma}\label{lem:3.1_gntk} 
We write $\H^{\cts}, \H^{\dis} \in \R^{n \times n}$ as follows
\begin{align*}
\H^{\cts}_{i,j} = & ~ \E_{w \sim \N(0,I)} \left[ \sum_{l_1,l_2} x_{i,l_1}^\top x_{j,l_2} {\bf 1}_{ w^\top x_{i,l_1} \geq 0, w^\top x_{j,l_2} \geq 0 } \right] , \\ 
\H^{\dis}_{i,j} = & ~ \frac{1}{m} \sum_{r=1}^m \left[ \sum_{l_1,l_2} x_{i,l_1}^\top x_{j,l_2} {\bf 1}_{ w_r^\top x_{i,l_1} \geq 0, w_r^\top x_{j,l_2} \geq 0 } \right].
\end{align*}
Let $\lambda = \lambda_{\min} ( \H^{\cts} ) $. If $m = \Omega( \lambda^{-2} N^2 n^2 \log (n/\delta) )$, we have 
\begin{align*}
\| \H^{\dis} - \H^{\cts} \|_F \leq \frac{ \lambda }{4}, \mathrm{~and~} \lambda_{\min} ( \H^{\dis} ) \geq \frac{3}{4} \lambda.
\end{align*}
hold with probability at least $1-\delta$.
\end{lemma}

\begin{proof} 
For every fixed pair $(i,j)$,
$\H_{i,j}^{\dis}$ is an average of independent random variables,
i.e.
\begin{align*}
\H_{i,j}^{\dis}=~\frac {1}{m} \sum_{r=1}^m \sum_{l_1 = 1}^N \sum_{l_2=1}^N x_{i,l_1}^\top x_{j,l_2} \mathbf{1}_{w_r^\top x_{i,l_1} \geq 0,w_r^\top x_{j,l_2} \geq 0}.
\end{align*}
Then the expectation of $\H_{i,j}^{\dis}$ is
\begin{align*}
& ~ \E [ \H_{i,j}^{\dis} ] \\
= & ~\frac {1}{m}\sum_{r=1}^m\sum_{l_1 = 1}^N \sum_{l_2=1}^N \E_{w_r\sim {\N}(0,I_d)} \left[ x_{i,l_1}^\top x_{j,l_2} \mathbf{1}_{w_r^\top x_{i,l_1} \geq 0,w_r^\top x_{j,l_2} \geq 0} \right]\\
= & ~\E_{w\sim {\N}(0,I_d)} \sum_{l_1 = 1}^N \sum_{l_2=1}^N \left[ x_{i,l_1}^\top x_{j,l_2} \mathbf{1}_{w_r^\top x_{i,l_1} \geq 0,w_r^\top x_{j,l_2} \geq 0} \right]\\
= & ~ \H_{i,j}^{\cts}.
\end{align*}
For $r \in [m]$,
we define $z_r$
\begin{align*}
z_r := \frac{1}{m} \sum_{l_1,l_2} x_{i,l_1}^\top x_{j,l_2} \mathbf{1}_{w_r^\top x_{i,l_1} \geq 0,w_r^\top x_{j,l_2} \geq 0}.
\end{align*}
Then $z_r$ is a random function of $w_r$,
hence $\{z_r\}_{r\in [m]}$ are mutually independent.
Moreover,
$-\frac {1}{m}  N^2 \leq z_r\leq \frac {1}{m} N^2$.  
So by Hoeffding inequality (Lemma \ref{lem:hoeffding}) we have for all $t>0$,
\begin{align*}
\Pr \left[ | \H_{i,j}^{\dis} - \H_{i,j}^{\cts} | \geq t \right]
\leq & ~ 2\exp \Big( -\frac{2t^2}{4N^2/m} \Big) \\
 = & ~ 2\exp(-mt^2/ (2N^2)).
\end{align*}
Setting $t=( \frac{1}{m} 2 N^2 \log (2n^2/\delta) )^{1/2}$,
we can apply union bound on all pairs $(i,j)$ to get with probability at least $1-\delta$,
for all $i,j\in [n]$,
\begin{align*}
| \H_{i,j}^{\dis} - \H_{i,j}^{\cts} |
\leq & ~ \Big( \frac{2}{m} N^2\log (2n^2/\delta) \Big)^{1/2} \\
\leq & ~ 4 \Big( \frac{N^2 \log ( n/\delta ) }{m} \Big)^{1/2}.
\end{align*}
Thus we have
\begin{align*}
\| \H^{\dis} - \H^{\cts} \|^2 
\leq & ~ \| \H^{\dis} - \H^{\cts} \|_F^2 \\
 = & ~ \sum_{i=1}^n\sum_{j=1}^n | \H_{i,j}^{\dis} - \H_{i,j}^{\cts} |^2 \\
 \leq & ~ \frac{1}{m} 16 N^2 n^2\log (n/\delta).
\end{align*}
Hence if $m=\Omega( \lambda^{-2} N^2 n^2\log (n/\delta) )$ we have the desired result.
 \end{proof}

Similarly, we can show
\begin{lemma}\label{lem:lemma_4.2_in_sy19}
Let
\begin{itemize}
    \item $R_0 \in (0,1)$.
\end{itemize}
 If the following conditions hold:
 \begin{itemize}
     \item $\wt{w}_1, \cdots, \wt{w}_m$ are i.i.d. generated from ${\cal N}(0,I)$.
     \item   For any set of weight vectors $w_1, \cdots, w_m \in \R^d$ that satisfy for any $r \in [m]$, $\| \wt{w}_r - w_r \|_2 \leq R_0$,
 \end{itemize}
 then the $H : \R^{ m \times d} \rightarrow \R^{n \times n}$ defined
\begin{align*}
    [\H(W)]_{i,j} = \frac{1}{m}  \sum_{r=1}^m \sum_{l_1 = 1}^N \sum_{l_2 = 1}^N x_{i,l_1}^\top x_{j,l_2} {\bf 1}_{ w_r^\top x_{i,l_1} \geq 0, w_r^\top x_{j,l_2} \geq 0 } .
\end{align*}
Then we have
\begin{align*}
    \| \H(w) - \H(\wt{w}) \|_F < 2 N n R_0
\end{align*}
holds with probability at least $1- N^2n^2  \cdot \exp(-m R_0 /10)$.
\end{lemma}

\subsection{Upper bounding initialization perturbation}\label{sec:equiv_intialization}
In this section we bound $\epsilon_{\init} \leq \epsilon$ by choosing a large enough $\kappa$. Our goal is to prove Lemma~\ref{lem:epsilon_init}. This result indicates that initially the predictor's output is small. This result can be proved by Gaussian tail bounds. 

\begin{lemma}[Bounding initialization perturbation]\label{lem:epsilon_init}
Let $f_{\gnn}$ be as defined in Definition~\ref{def:f_nn}. Assume
\begin{itemize}
    \item the initial weight of the network work $w_r(0)\in\R^d,~r=1,\cdots,m$ as defined in Definition~\ref{def:nn} are drawn independently from standard Gaussian distribution $\mathcal{N}(0,I_d)$.
    \item and $a_r\in\R,~r=1,\cdots,m$ as defined in Definition~\ref{def:f_nn} are drawn independently from $\unif[\{-1,+1\}]$.
\end{itemize}
  Let
  \begin{itemize}
      \item $\kappa\in(0,1)$, $u_{\gnn}(t)$ and $u_{\gnn,\test}(t) \in \R$ be defined as in Definition~\ref{def:nn}.
  \end{itemize}
   For any data $G$, assume its corresponding  $\{ x_{l} \}_{l \in [N]} \in \R^d$ satisfying that $\|x_{l}\|_2 \leq R, \forall l\in [N]$. Then, we have with probability $1-\delta$, 
\begin{align*}
	|f_{\gnn}(W(0),G)| \leq 2 N R \log(2 N m/\delta).
\end{align*}
Further, given any accuracy $\epsilon\in(0,1)$, if $\kappa = \wt{O}(\epsilon \Lambda_0 / ( N R n ) )$, let $\epsilon_{\init} = \epsilon \Lambda_0 / ( N R n )$, we have
\begin{align*}
	|u_{\gnn}(0)| \leq \sqrt{n}\epsilon_{\init}~\text{and}~|u_{\gnn,\test}(0)| \leq \epsilon_{\init}
\end{align*}
hold with probability $1-\delta$, where $\wt{O}(\cdot)$ hides the $\poly\log( N n / ( \epsilon  \delta  \Lambda_0 ) )$.
\end{lemma}
\begin{proof}
Note by definition,
\begin{align*}
	f_{\gnn}(W(0),G) = \frac{1}{\sqrt{m}}\sum_{r=1}^m a_r \sum_{l=1}^N \sigma(w_r(0)^\top x_l).
\end{align*}
Since $w_r(0)\sim\mathcal{N}(0,I_d)$, so $w_r(0)^\top x_{\test}\sim N(0,\|x_{\test,l}\|_2)$. Note $\|x_{\test,l}\|_2 \leq R$, $\forall l \in [N]$, by Gaussian tail bounds~Lemma~\ref{lem:gaussian_tail}, we have with probability $1-\delta / (2 N m)$:
\begin{align}\label{eq:guassian_tail}
	|w_r(0)^\top x| \leq R \sqrt{ 2 \log( 2 N m  / \delta ) }.
\end{align}
Condition on Eq.~\eqref{eq:guassian_tail} holds for all $r\in[m]$, denote $Z_r = a_r \sum_{l=1}^N \sigma(w_r(0)^\top x_l)$, then we have $\E[Z_r] = 0$ and $|Z_r| \leq N R \sqrt{2\log(2m/\delta)}$. By Lemma~\ref{lem:hoeffding}, with probability $1-\delta/2$:
\begin{align}\label{eq:heoffding}
	\Big| \sum_{r=1}^m Z_r \Big| \leq 2 NR \sqrt{m}\log{ ( 2 N m / \delta ) }.
\end{align}
Since $u_{\gnn,\test}(0) = \frac{1}{\sqrt{n}}\sum_{r=1}^m Z_r $, by combining Eq.~\eqref{eq:guassian_tail},~\eqref{eq:heoffding} and union bound over all $r \in [m]$, we have with probability $1-\delta$:
\begin{align*}
	|u_{\gnn,\test}(0)| \leq 2 N R \log( 2 N m / \delta ).
\end{align*}
Further, note $[u_{\gnn}(0)]_i = \kappa f_{\gnn}(W(0), G_i)$ and $u_{\gnn,\test}(0) = \kappa f_{\gnn}(W(0), G_{\test} )$. Thus, by choosing 
\begin{align*}
\kappa = \wt{O}( \epsilon \Lambda_0 / ( N R n ) ),
\end{align*}
 taking the union bound over all training and test data, we have
\begin{align*}
	|u_{\gnn}(0)| \leq \sqrt{n}\epsilon_{\init}~\text{and}~|u_{\gnn,\test}(0)| \leq \epsilon_{\init}
\end{align*}
hold with probability $1-\delta$, where $\wt{O}(\cdot)$ hides the $\poly\log( N n / ( \epsilon  \delta  \Lambda_0 ) )$.
\end{proof}

\subsection{Upper bounding kernel perturbation}\label{sec:equiv_kernel_perturbation}
In this section, our goal is to prove Lemma~\ref{lem:induction}. We bound the kernel perturbation using induction. More specifically, we prove four induction lemma. Each of them states a property still holds after one iteration. We combine them together and prove  Lemma~\ref{lem:induction}. 

\begin{lemma}[Bounding kernel perturbation]\label{lem:induction}
Given 
\begin{itemize}
    \item training data $\G\in\R^{n\times d \times N}$, $Y\in\R^n$ and a test data $G_{\test}\in\R^d$.
\end{itemize}
 Let 
 \begin{itemize}
     \item  $T > 0$ denotes the total number of iterations, 
     \item $m >0 $ denotes the width of the network,
     \item $\epsilon_{\train}$ denotes a fixed training error threshold,
     \item  $\delta > 0$ denotes the failure probability.
     \item $u_{\gnn}(t) \in \R^n$ and $u_{\gntk}(t) \in \R^n$ be the training data predictors defined in Definition~\ref{def:nn} and Definition~\ref{def:krr_ntk} respectively. 
     \item $\kappa\in(0,1)$ be the corresponding multiplier.
     \item  $\k_{\gntk}( G_{\test}, \G ) \in \R^n,~\k_{t}( G_{\test}, \G ) \in \R^n,~ \H(t) \in \R^{n \times n},~\Lambda_0 > 0$ be the kernel related quantities defined in Definition~\ref{def:ntk_phi} and Definition~\ref{def:dynamic_kernel}.
     \item $u^* \in \R^n$ be defined as in Eq.~\eqref{eq:def_u_*}.
     \item $W(t) = [w_1(t),\cdots,w_m(t)]\in\R^{d\times m}$ be the parameters of the neural network defined in Definition~\ref{def:nn}.
 \end{itemize}

For any accuracy $\epsilon\in(0,1/10)$. If
\begin{itemize}
    \item $\kappa=\wt{O}(\frac{\epsilon\Lambda_0}{N R n})$, $T=\wt{O}(\frac{1}{\kappa^2 \Lambda_0 })$, $\epsilon_{\train} = \wt{O}(\|u_{\gnn}(0)-u^*\|_2)$, $m \geq\wt{O}(\frac{N^2 n^{10} d}{\epsilon^6 \Lambda_0^{10}})$, with probability $1-\delta$,
\end{itemize}
 there exist $\epsilon_W,~\epsilon_H',~\epsilon_K'>0$ that are independent of $t$, such that the following hold for all $0 \leq t \le T$:
\begin{itemize}
    \item 1. $\| w_r(0) - w_r(t) \|_2 \leq \epsilon_W $, $\forall r \in [m]$
    \item 2. $\| \H(0) - \H(t) \|_2 \leq \epsilon_H'$ 
    \item 3. $\| u_{\gnn}(t) - u^* \|_2^2 \leq \max\{\exp(-(\kappa^2\Lambda_0 ) t/2) \cdot \| u_{\gnn}(0) - u^* \|_2^2, ~ \epsilon_{\train}^2\}$
    \item 4. $\| \k_0( G_{\test} , \G )- \k_{t} ( G_{\test}, \G ) \|_2 \leq \epsilon_K'$
\end{itemize}
Further, 
\begin{align*}
\epsilon_W \leq \wt{O}(\frac{\epsilon \Lambda_0^2}{n^2}), ~~~  \epsilon_H' \leq \wt{O}(\frac{\epsilon \Lambda_0^2}{n}) \mathrm{~~~and~~~} \epsilon_K' \leq \wt{O}(\frac{\epsilon \Lambda_0^2}{n^{1.5}}).
\end{align*}

Here $\wt{O}(\cdot)$ hides the  $\poly\log( N n / ( \epsilon  \delta  \Lambda_0) )$.
\end{lemma}

In order to prove this lemma, we first state some helpful concentration results for random initialization. 

\begin{lemma}[Random initialization result]\label{lem:random_init}
Assume
\begin{itemize}
    \item initial value $w_r(0) \in \R^d ,~r=1,\cdots,m$ are drawn independently from standard Gaussian distribution $\mathcal{N}(0,I_d)$,
\end{itemize}
 then with probability $1-3\delta$ we have
\begin{align}
& ~	\|w_r(0)\|_2 \leq  2\sqrt{d} + 2\sqrt{\log{(m/\delta)}}~\text{for all}~r\in[m]\label{eq:3322_1}\\
& ~	\| \H(0)- \H^{\cts} \| \leq  4 N n ( \log(n/\delta) / m )^{1/2}\label{eq:3322_2}\\
 & ~	\|\k_{0}( G_{\test} , \G ) - \k_{\gntk} ( G_{\test} , \G )\|_2 \leq ( 2 N^2 n \log{(2n/\delta)} / m )^{1/2}\label{eq:3322_3}
\end{align}
\end{lemma}
\begin{proof}
By standard concentration inequality, with probability at least $1-\delta$,
\begin{align*}
	\| w_r(0) \|_2 \leq \sqrt{d} + \sqrt{\log(m/\delta)}
\end{align*}
holds for all $r\in[m]$.\\

Using Lemma~\ref{lem:3.1_gntk}, we have
\begin{align*}
    \| H(0) - H^{\cts} \| \leq \epsilon_H'' = 4 N n ( \log{(n/\delta)} / m )^{1/2}
\end{align*}
holds with probability at least $1-\delta$.\\
Note by definition,
\begin{align*}
    \E[\k_0 ( G_{\test}, x_i )] =  \k_{\gntk} ( G_{\test}, G_i )
\end{align*}
holds for any training data $G_i$. By Hoeffding inequality, we have for any $t>0$,
\begin{align*}
    \Pr[|\k_0 ( G_{\test}, G_i ) - \k_{\gntk} ( G_{\test}, G_i )|\ge t] \le 2\exp{(-mt^2/(2N^2))}.
\end{align*}
Setting $t=(\frac{2}{m} N^2 \log{(2n/\delta)})^{1/2}$, we can apply union bound on all training data $G_i$ to get with probability at least $1-\delta$, for all $i\in[n]$,
\begin{align*}
    |\k_0 ( G_{\test}, G_i ) - \k_{\gntk} ( G_{\test}, G_i )| \le (2 N^2 \log(2n/\delta) / m)^{1/2}.
\end{align*}
Thus, we have
\begin{align}
    \| \k_0 ( G_{\test}, \G ) - \k_{\gntk} ( G_{\test}, \G ) \|_2 \le ( 2 N^2 n\log(2n/\delta) / m )^{1/2}
\end{align}
holds with probability at least $1-\delta$.\\
Using union bound over above three events, we finish the proof.
\end{proof}

Given the conditions set by Eq.\eqref{eq:3322_1},\eqref{eq:3322_2}, and~\eqref{eq:3322_3}, we demonstrate that all four outcomes outlined in Lemma~\ref{lem:induction} are valid, employing an inductive approach.

We define the following quantity:
\begin{align}
	\epsilon_W :=  & ~ \frac{ \sqrt{Nn} }{ \sqrt{m} } \max\{4\| u_{\gnn}(0) - u^* \|_2/(\kappa^2\Lambda_0 ), \epsilon_{\train} \cdot T\} \label{eq:def_epsilon_W} \\
    \epsilon_H' := & ~ 2n\epsilon_W\notag \\
    \epsilon_K := & ~ 2\sqrt{n}\epsilon_W\notag
\end{align}
which are independent of $t$. 

Note that the base case where $t=0$ trivially holds. Under the induction hypothesis that Lemma~\ref{lem:induction} holds before time $t\in[0,T]$, we will prove that it still holds at time $t$. We prove this in Lemma~\ref{lem:hypothesis_1},~\ref{lem:hypothesis_2},~and~\ref{lem:hypothesis_3}.


\begin{lemma}\label{lem:hypothesis_1}
If for any $\tau < t$, we have
\begin{align*}
    \| &~u_{\gnn}(\tau) - u^* \|_2^2 \\
    \leq &~ \max\{\exp(-(\kappa^2 \Lambda_0 ) \tau/2) \cdot \| u_{\gnn}(0) - u^* \|_2^2,~\epsilon_{\train}^2\}
\end{align*}
and
\begin{align*}
     \| w_r(0) - w_r(\tau) \|_2 \leq \epsilon_W \leq 1
\end{align*}
and
\begin{align*}
	\| w_r(0) \|_2 \leq ~ \sqrt{d} + \sqrt{\log(m/\delta)} ~ \text{for all}~r\in[m]
\end{align*}
hold,
then
\begin{align*}
    \| w_r(0) - w_r(t) \|_2 \leq \epsilon_W 
\end{align*}

\end{lemma}
\begin{proof}

For simplicity, we consider the case when $N=1$. For general case, we need to blow up an $N$ factor in the upper bound.

Recall the gradient flow as Eq.~\eqref{eq:323_1}
\begin{align}\label{eq:332_2}
    \frac{ \d w_r( \tau ) }{ \d \tau } = ~ \sum_{i=1}^n \frac{1}{\sqrt{m}} a_r ( y_i - u_{\gnn}(\tau)_i ) x_i \sigma'( w_r(\tau)^\top x_i ) 
\end{align}
So we have
\begin{align}
	\Big\| \frac{ \d w_r( \tau ) }{ \d \tau } \Big\|_2 \label{eq:332_1}
	= & ~ \left\| \sum_{i=1}^n \frac{1}{\sqrt{m}} a_r ( y_i - u_{\gnn}(\tau)_i ) x_i \sigma'( w_r(\tau)^\top x_i ) \right\|_2  \\ 
	\leq & ~ \frac{1}{\sqrt{m}} \sum_{i=1}^n |y_i-u_{\gnn}(\tau)_i| \notag \\ 
	\leq & ~ \frac{ \sqrt{n} }{ \sqrt{m} } \| Y-u_{\gnn}(\tau) \|_2 \notag \\ 
	\leq & ~ \frac{ \sqrt{n} }{ \sqrt{m} } (\| Y-u^*\|_2 + \| u_{\gnn}(\tau) - u^*\|_2) \notag \\ 
	\leq & ~ \frac{ \sqrt{n} }{ \sqrt{m} } \max\{e^{-(\kappa^2 \Lambda_0)\tau/4} \| u_{\gnn}(0) - u^* \|_2, \epsilon_{\train}\} \notag  
\end{align}
where the initial step is based on Eq.\eqref{eq:332_2}. The subsequent step arises from the triangle inequality. The third step is derived from the Cauchy-Schwarz inequality. The fourth step again references the triangle inequality. The concluding step is informed by the condition $| u_{\gnn}(\tau) - u^* |2^2 \leq \max{\exp(-(\kappa^2\Lambda_0) \tau/2) \cdot | u{\gnn}(0) - u^* |_2^2,\epsilon_{\train}^2}$.

Thus, for any $t \le T$,
\begin{align*}
	& ~ \| w_r(0) - w_r(t) \|_2 \\
	\leq & ~ \int_0^t \Big\| \frac{ \d w_r( \tau ) }{ \d \tau } \Big\|_2 d\tau \\
	\leq & ~ \frac{ \sqrt{Nn} }{ \sqrt{m} } \max\{4\| u_{\gnn}(0) - u^* \|_2/(\kappa^2\Lambda_0), \epsilon_{\train}\cdot T \} \\
    = & ~ \epsilon_W
\end{align*}
where the initial step arises from the triangle inequality. The subsequent step is based on Eq.\eqref{eq:332_1}. The concluding step is informed by the definition of $\epsilon_W$, as outlined in Eq.\eqref{eq:def_epsilon_W}.
\end{proof}

\begin{lemma}\label{lem:hypothesis_2}
If $\forall r \in [m]$,
\begin{align*}
    \| w_r(0) - w_r(t) \|_2 \leq \epsilon_W < 1,
\end{align*}
then
\begin{align*}
    \| \H(0) - \H(t) \|_F \leq 2 n \epsilon_W
\end{align*}
holds with probability $1- N^2 n^2 \cdot \exp{(-m\epsilon_W/(10N^2))}$.
\end{lemma}
\begin{proof}
Applying Lemma~\ref{lem:lemma_4.2_in_sy19} completes the proof.
\end{proof}


\begin{lemma}\label{lem:hypothesis_3}
we have
\begin{align*}
    &~ \| u_{\gnn}(t) - u^* \|_2^2\\
    \leq &~ \max\{\exp(-(\kappa^2\Lambda_0 )t/2 ) \cdot \| u_{\gnn}(0) - u^* \|_2^2, ~ \epsilon_{\train}^2\}.
\end{align*}
\end{lemma}

\begin{proof}
By Lemma~\ref{lem:linear_converge_nn}, for any $\tau < t$, we have
\begin{align}\label{eq:induction_linear_convergence}
    \frac{ \d \| u_{\gnn}(\tau) - u^* \|_2^2 }{ \d \tau } 
    \leq & ~ - 2( \kappa^2 \Lambda_0 ) \cdot \| u_{\gnn} (\tau) - u^* \|_2^2 
\end{align}
where the step follows from Lemma~\ref{lem:linear_converge_nn}, which implies
    
    \begin{align*}
		\| u_{\gnn}(t) - u^* \|_2^2 \leq \exp{(-2(\kappa^2 \Lambda_0 )t)} \cdot \| u_{\gnn}(0) - u^* \|_2^2.
	\end{align*}

we conclude
\begin{align*}
&~\| u_{\gnn}(t) - u^* \|_2^2\\
\leq &~\max\{\exp(-(\kappa^2\Lambda_0 )t/2 ) \cdot \| u_{\gnn}(0) - u^* \|_2^2, ~ \epsilon_{\train}^2\}.
\end{align*}
\end{proof}


\begin{lemma}\label{lem:hypothesis_4}
Fix $\epsilon_W\in(0,1)$ independent of $t$. If $\forall r \in [m]$, we have
\begin{align*}
    \| w_r(t) - w_r(0) \|_2 \leq \epsilon_W
\end{align*}
then
\begin{align*}
    \| \k_t( G_{\test} , \G ) - \k_{0} ( G_{\test} , \G ) \|_2 \leq \epsilon_K' = 2\sqrt{n}\epsilon_W
\end{align*}
holds with probability at least $1- N^2\cdot n\cdot\exp{(-m\epsilon_W/(10N^2))}$.
\end{lemma}
\begin{proof}

Recall the definition of $\k_0$ and $\k_t$
\begin{align*}
    &~\k_0 ( G_{\test}, G_i) \\= & ~ \frac{1}{m} \sum_{r=1}^m \sum_{l_1=1}^N \sum_{l_2=1}^N x_{\test,l_1}^\top x_{i,l_2} \sigma'( x_{\test,l_1}^\top w_r(0) ) \sigma'( x_{i,l_2}^\top w_r(0) ) \\
    &~\k_t ( G_{\test}, G_i)\\
    = & ~ \frac{1}{m} \sum_{r=1}^m \sum_{l_1=1}^N \sum_{l_2=1}^N x_{\test,l_1}^\top x_{i,l_2} \sigma'( x_{\test,l_1}^\top w_r(t) ) \sigma'( x_{i,l_2}^\top w_r(t) )
\end{align*}
By direct calculation we have 
\begin{align*}
    \| \k_0 ( G_{\test}, \G ) - \k_t ( x_{\test}, X ) \|_2^2\le \sum_{i=1}^n \Big( \frac{1}{m}\sum_{r=1}^m s_{r,i} \Big)^2,
\end{align*}
where 
\begin{align*}
s_{r,i} = & ~ \sum_{l_1=1}^N \sum_{l_2=1}^N {\bf 1}[w_r(0)^\top x_{\test,l_1} \ge 0, w_r(0)^\top x_{i,l_2} \ge 0] \\
& ~ - {\bf 1}[w_r(t)^\top x_{\test,l_1} \ge 0, w_r(t)^\top x_{i,l_2} \ge 0], \\
& ~ \forall r \in [m], i \in [n].
\end{align*}
Fix $i\in [n]$, by Bernstein inequality (Lemma~\ref{lem:bernstein}), we have for any $t>0$, 
\begin{align*}
    & ~ \Pr \Big[ \frac{1}{m}\sum_{r=1}^m s_{r,i} \ge 2\epsilon_W \Big] \\
    \le & ~ \exp ( - m \epsilon_W / (10N^2) ) .
\end{align*}
Thus, applying union bound over all training data $x_i,~i\in[n]$, we conclude
\begin{align*}
    & ~ \Pr[\| \k_0 ( x_{\test}, X ) - \k_t ( x_{\test}, X ) \|_2 \le 2\sqrt{n}\epsilon_W] \\
    \le & ~ 1-n\cdot\exp{(-m\epsilon_W/(10N^2))}.
\end{align*}
Note by definition $\epsilon_K' = 2\sqrt{n} \epsilon_W$, so we finish the proof.

\end{proof}

Now, we choose the parameters as follows: $\kappa=\wt{O}(\frac{\epsilon\Lambda_0}{N R n}) ,~T=\wt{O}(\frac{1}{\kappa^2 \Lambda_0 })$, $\epsilon_{\train} = \wt{O}(\|u_{\gnn}(0)-u^*\|_2)$, $m \geq\wt{O}(\frac{N^2n^{10}d}{\epsilon^6 \Lambda_0^{10}})$.

Further, with probability $1-\delta$, we have
\begin{align*}
	\|u_{\gnn}(0) - u^*\|_2 \leq & ~ \|u_{\gnn}(0)\|_2 + \|Y-u^*\|_2 + \|Y\|_2 \\
	\leq & ~ \wt{O}(\frac{\epsilon\Lambda_0}{\sqrt{n}}) + 0 + \wt{O}(\sqrt{n}) \\ 
	\leq & ~ \wt{O}(\sqrt{n})
\end{align*}
where the first step follows from triangle inequality, the second step follows from Lemma~\ref{lem:epsilon_init}, 
and $Y=O(\sqrt{n})$, and the last step follows $\epsilon,~\Lambda_0<1$. With same reason,
\begin{align*}
	\|u_{\gnn}(0) - u^*\|_2 \geq & ~ -\|u_{\gnn}(0)\|_2 - \|Y-u^*\|_2 + \|Y\|_2 \\
	\geq & ~ -\wt{O}(\frac{\epsilon\Lambda_0}{\sqrt{n}}) - 0 + \wt{O}(\sqrt{n}) \\ 
	\geq & ~ \wt{O}(\sqrt{n}).
\end{align*}
Thus, we have $\|u_{\gnn}(0) - u^*\|_2=\wt{O}(\sqrt{n})$. 

We have proved that with high probability all the induction conditions are satisfied. Note that the failure probability only comes from Lemma~\ref{lem:epsilon_init},~\ref{lem:random_init},~\ref{lem:hypothesis_2},~\ref{lem:hypothesis_4}, which only depends on the initialization. By union bound over these failure events, we have that with high probability all the four statements in Lemma~\ref{lem:induction} are satisfied. This completes the proof. 




\subsection{Connect iterative GNTK regression and GNN training process}\label{sec:equiv_bound_nn_test_T_and_ntk_test_T} 
In this section we prove Lemma~\ref{lem:equivalence_at_T}. Lemma~\ref{lem:equivalence_at_T} bridges GNN test predictor and GNTK test predictor. We prove this result by bound the difference between dynamic kernal and GNTK, which can be further bridged by the graph dynamic kernal at initial. 
\begin{lemma}[Upper bounding test error]\label{lem:equivalence_at_T}
Given
\begin{itemize}
    \item training graph data $\G = \{G_1, \cdots, G_n \}$ and corresponding label vector $Y \in \R^n$.
    \item  the total number of iterations $T > 0$.
    \item arbitrary test data $G_{\test} \in \R^d$.
\end{itemize}
 Let 
 \begin{itemize}
     \item  $u_{\gnn,\test}(t) \in \R^n$ and $u_{\gntk,\test}(t) \in \R^n$ be the test data predictors defined in Definition~\ref{def:nn} and Definition~\ref{def:krr_ntk} respectively.
     \item $\kappa\in(0,1)$ be the corresponding multiplier.
 \end{itemize}
  Given accuracy $\epsilon>0$, if
  \begin{itemize}
      \item   $\kappa = \wt{O}(\frac{\epsilon\Lambda_0}{NRn })$, $T=\wt{O}(\frac{1}{\kappa^2\Lambda_0})$, $m \geq \wt{O}(\frac{N^2 n^{10}d}{\epsilon^6\Lambda_0^{10}})$ . 
  \end{itemize}
Then for any $G_{\test}\in\R^d$, with probability at least $1-\delta$ over the random initialization, we have
\begin{align*}
\| u_{\gnn,\test}(T) - u_{\gntk,\test}(T) \|_2 \leq \epsilon/2,
\end{align*}
where $\wt{O}(\cdot)$ hides $\poly\log( n / (\epsilon\delta\Lambda_0) )$.
\end{lemma}
\begin{proof}
By Lemma~\ref{lem:more_concreate_bound}, we have
\begin{align}\label{eq:b44_1}
	|u_{\gnn,\test}(T)-u_{\gntk,\test}(T)| \leq & ~ (1+\kappa^2 nT)\epsilon_{\init} + \epsilon_K \cdot \frac{ \| u^* \|_2 }{  \Lambda_0 } \notag \\
    & ~ +   \sqrt{n}T^2\kappa^4 \epsilon_H \| u^* \|_2 
\end{align}
By Lemma~\ref{lem:epsilon_init}, we can choose $\epsilon_{\init} = \epsilon(\Lambda_0)/n$. 

Further, note 
\begin{align*}
	& ~ \|\k_{\gntk}( G_{\test}, \G )-\k_{t}( G_{\test},\G)\|_2 \\
	\leq & ~ \|\k_{\gntk}( G_{\test},\G )-\k_0(x_{\test},X)\|_2 \\
	&\quad+ \|\k_0( G_{\test},\G)-\k_{t}( G_{\test},\G)\|_2 \\
	\leq & ~  ( 2N^2n \log{(2n/\delta)} / m )^{1/2} + \|\k_0( G_{\test},\G)-\k_{t}( G_{\test},\G)\|_2 \\
	\leq & ~  \wt{O}(\frac{\epsilon \Lambda_0^2}{n^{1.5}})
\end{align*}
where the initial step arises due to the triangle inequality. The subsequent step is derived from Lemma~\ref{lem:random_init}. The concluding step is based on Lemma~\ref{lem:induction}. Consequently, we can select $\epsilon_K = \frac{\epsilon \Lambda_0^2}{n^{1.5}}$.

Also,
\begin{align*}
	& ~ \| \H^{\cts}- \H(t)\| \\
	\leq & ~ \| \H^{\cts} - \H(0)\| + \| \H(0)- \H(t)\|_2 \\
	\leq & ~  4N n ( \log(n/\delta) / m )^{1/2} + \| \H(0) - \H(t)\|_2 \\
	\leq & ~  \wt{O}(\frac{\epsilon \Lambda_0^2}{n})
\end{align*}
where the initial step is informed by the triangle inequality. The next step is based on Lemma~\ref{lem:random_init}. The final step is derived from Lemma~\ref{lem:induction}. As a result, we can set $\epsilon_H = \frac{\epsilon \Lambda_0^2}{n}$.

Note $\|u^*\|_2 \leq \sqrt{n}$, plugging the value of $\epsilon_{\init},~\epsilon_K,~\epsilon_H$ into Eq.~\eqref{eq:b44_1}, we have 
\begin{align*}
	|u_{\gnn,\test}(T)-u_{\gntk,\test}(T)| \leq \epsilon/2.
\end{align*}
\end{proof}

\begin{lemma}[Upper bounding test error (node version)]\label{lem:equivalence_at_T_node}
Given 
\begin{itemize}
    \item training graph $G$ and corresponding label vector $Y \in \R^n$.
    \item  $T > 0$ total number of iterations.
    \item arbitrary test node $v \in \R^d$.
\end{itemize}
  Let
  \begin{itemize}
      \item  $u_{\gnn,\test,\node}(t) \in \R^n$ and $u_{\gntk,\test,\node}(t) \in \R^n$ be the test data predictors defined in Definition~\ref{def_main:nn_node} and Definition~\ref{def_main:krr_ntk_node} respectively.
      \item  $\kappa\in(0,1)$ be the corresponding multiplier. 
  \end{itemize}
 Given accuracy $\epsilon>0$, if
 \begin{itemize}
     \item  $\kappa = \wt{O}(\frac{\epsilon\Lambda_0}{RN })$, $T=\wt{O}(\frac{1}{\kappa^2\Lambda_0})$, $m \geq \wt{O}(\frac{ N^{10}d}{\epsilon^6\Lambda_0^{10}})$ .
 \end{itemize}
Then with probability at least $1-\delta$ over the random initialization, we have
\begin{align*}
\| u_{\gnn,\test,\node}(T) - u_{\gntk,\test,\node}(T) \|_2 \leq \epsilon/2,
\end{align*}
where $\wt{O}(\cdot)$ hides $\poly\log( N / (\epsilon\delta\Lambda_0) )$.
\end{lemma}

\subsection{Main result}\label{sec:equiv_main_test_equivalence} 
In this section we present our main theorem.
First, we present the equivalence of graph level between training GNN and GNTK regression for test data prediction. Second, we present the equivalence of node level between training GNN and GNTK regression for test data prediction. 

\begin{theorem}[Equivalence between training net with and kernel regression for test data prediction]\label{thm:main_test_equivalence}
Given
\begin{itemize}
    \item training graph data $\G = \{ G_1, \cdots, G_n \}$ and corresponding label vector $Y \in \R^n$. Let $T > 0$ be the total number of iterations.
    \item arbitrary test data $G_{\test}$. 
\end{itemize}
 Let
 \begin{itemize}
     \item  $u_{\gnn,\test}(t) \in \R^n$ and $u_{\test}^* \in \R^n$ be the test data predictors defined in Definition~\ref{def:nn} and Definition~\ref{def:krr_ntk} respectively.
 \end{itemize}
For any accuracy $\epsilon \in (0,1/10)$ and failure probability $\delta \in (0,1/10)$, if
\begin{itemize}
    \item $\kappa = \wt{O}(\frac{\epsilon\Lambda_0}{NRn})$, $T=\wt{O}(\frac{1}{\kappa^2\Lambda_0})$, $m \geq \wt{O}(\frac{N^2 n^{10}d}{\epsilon^6\Lambda_0^{10}})$.
\end{itemize}
 Then for any $G_{\test}$, with probability at least $1-\delta$ over the random initialization, we have
\begin{align*}
\| u_{\gnn,\test}(T) - u_{\test}^* \|_2 \leq \epsilon.
\end{align*}
Here $\wt{O}(\cdot)$ hides $\poly\log(Nn/(\epsilon \delta \Lambda_0 ))$.
\end{theorem}
\begin{proof}
It follows from combining results of bounding $\| u_{\gnn,\test}(T) - u_{\gntk,\test}(T) \|_2 \leq \epsilon/2$ as shown in Lemma~\ref{lem:equivalence_at_T} and $\| u_{\gntk,\test}(T) - u_{\test}^* \|_2 \leq \epsilon/2$ as shown in Lemma~\ref{lem:u_ntk_test_T_minus_u_test_*} using triangle inequality.
\end{proof}

\begin{theorem}[Equivalence result for node level regression, formal version of \ref{thm_main:main_test_equivalence_node}]\label{thm:main_test_equivalence_node}
Given
\begin{itemize}
    \item training graph data $G$ and corresponding label vector $Y \in \R^N$.
    \item $T > 0$ be the total number of iterations.
    \item  arbitrary test data $v$.
\end{itemize}
 Let 
 \begin{itemize}
     \item  $u_{\gnn,\node,\test}(t) \in \R$ and $u_{\test,\node}^* \in \R$ be the test data predictors defined in Definition~\ref{def_main:nn_node} and Definition~\ref{def_main:krr_ntk} respectively. 
 \end{itemize}

For any accuracy $\epsilon \in (0,1/10)$ and failure probability $\delta \in (0,1/10)$, if 
\begin{itemize}
    \item the multiplier $\kappa = \wt{O}(N^{-1}\poly(\epsilon, \Lambda_{0,\node}, R^{-1}))$,
    \item  the total iterations $T=\wt{O}(N^{2}\poly(\epsilon^{-1}, \Lambda_{0,\node}^{-1}))$,
    \item  the width of the neural network $ m \geq \wt{O}(N^{10}\poly(\epsilon^{-1}, \Lambda_{0,\node}^{-1}, d)) $,
    \item  and the smallest eigenvalue of node level GNTK $\Lambda_{0,\node} > 0$, 
\end{itemize}

then for any $v_{\test}$, with probability at least $1-\delta$ over the random initialization, we have
\begin{align*}
   \| u_{\gnn,\test,\node}(T) - u_{\test,\node}^* \|_2 \leq \epsilon. 
\end{align*}

Here $\wt{O}(\cdot)$ hides $\poly\log(N/(\epsilon \delta \Lambda_{0,\node} ))$. 
\end{theorem}
\begin{proof}
It follows from combining results of bounding $\| u_{\gnn,\test,\node}(T) - u_{\gntk,\test,\node}(T) \|_2 \leq \epsilon/2$ as shown in Lemma~\ref{lem:equivalence_at_T_node} and $\| u_{\gntk,\test,\node}(T) - u_{\test,\node}^* \|_2 \leq \epsilon/2$ as shown in Lemma~\ref{lem:u_ntk_test_T_minus_u_test_*_node} using triangle inequality.
\end{proof}

\section{Bounds for the Spectral Gap with Data Separation}\label{sec:separation}

In this section, we provide a lower bound and an upper bound on the eigenvalue of the GNTK. This result can be proved by carefully decomposing probability. In Section~\ref{sec:separation:standard}, we consider the standard setting. In Section~\ref{sec:separation:shifted}, we consider the shifted NTK scenario.

\subsection{Standard setting}\label{sec:separation:standard}

In this section, we consider the GNTK defined in Definition \ref{def:ntk_phi}.

\begin{lemma}[Corollary I.2 of \cite{os20}]\label{PD-corollary}
We define $\mathcal{I}(z)={\bf 1}_{z\geq 0}$. Let 
\begin{itemize}
    \item $ x_1, \cdots, x_n$ be points in $\R^d$ with unit Euclidian norm and $ w\sim N(0, I_d)$.
    \item  Form the matrix $X \in R^{n\times d}=[x_1, \cdots, x_n]^{\top}$.
\end{itemize}
 Suppose there exists $\delta > 0$ such that for every $1\leq i\ne j\leq n $ we have that
\begin{align*}
    \min(\|x_i-x_j\|_2, \|x_i+x_j\|_2) \geq \delta.
\end{align*}
Then, the covariance of the vector $\mathcal{I}(Xw)$ obeys
\begin{align*}
    \E[\mathcal{I}(Xw)\mathcal{I}(Xw)^\top \odot XX^\top ] \succeq \frac{\delta}{100 n^2} 
\end{align*}

\end{lemma}

Then we prove our eigenvalue bound for GNTK.
\begin{theorem}[GNTK with generalized datasets $\delta$-separation assumption]
Let us define
\begin{align*}
     h_{G, u} = \sum_{ v \in \mathcal{N}(u) } h(G)_v, ~x_{i, p}=h_{G_i,u_p}, ~i\in[n], p\in[N] .
\end{align*}
Let $w_r$ be defined as in Definition \ref{def_main:f_nn}. Assuming that
\begin{itemize}
    \item  $w_r\overset{i.i.d.}{\sim} \N(0,I),~r=1, \cdots, m$. 
    \item our datasets satisfy the $\delta$-separation: For any $i,j \in[n], ~p,q\in[N]$
\begin{align*}
\min\{\| \ov{x}_{i, p} - \ov{x}_{j,q} \|_2, \| \ov{x}_{i, p} + \ov{x}_{j,q} \|_2    \}  \geq \delta .
\end{align*}
where $\ov{x}_{i,p}= x_{i,p}/\|x_{i,p}\|_2$.
\end{itemize}

We define $H^{\cts}$ as in Definition \ref{def_main:ntk_phi}. Then we can claim that 
\begin{align*}
H^{\cts} \succ \frac{\delta }{100 n^2} ,~ \Lambda_0 > \frac{\delta }{100 n^2}    
\end{align*}
\end{theorem}
\begin{proof}
Because
\begin{align*}
    & [\H^{\cts}]_{i,j}\\
    = &E_W \left[\frac{1}{m}  \sum_{r=1}^m \sum_{l_1 = 1}^N \sum_{l_2 = 1}^N x_{i,l_1}^\top x_{j,l_2} {\bf 1}_{ w_r^\top x_{i,l_1} \geq 0, w_r^\top x_{j,l_2} \geq 0 } \right].
\end{align*}
we will claim that
\begin{align*}
    \H^{\cts} =&~ \E_W \left[ \frac{1}{m}  \sum_{r=1}^m \sum_{l_1 = 1}^N \sum_{l_2 = 1}^N X_{l_1} X_{l_2}^\top \odot {\bf 1}_{ w_r^\top X_{l_1} \geq 0} {\bf 1}_{w_r^\top X_{l_2} \geq 0 } \right].\\
    =&~\frac{1}{m}  \sum_{r=1}^m \E_W \left[  (\sum_{l_1 =  1}^N X_{l_1}\odot {\bf 1}_{ w_r^\top X_{l_1} \geq 0} ) ( \sum_{l_2 = 1}^N X_{l_2}^\top \odot {\bf 1}_{w_r^\top X_{l_2} \geq 0 })    \right].
\end{align*}
where $X_{i}=[x_{1,i}, \cdots, x_{n,i}]^\top \in\R^{d \times n}$.
By Lemma \ref{PD-corollary}, we get that
\begin{align*}
   E_W\left[ X_{l_1} X_{l_2}^\top \odot {\bf 1}_{ w_r^\top X_{l_1} \geq 0} {\bf 1}_{w_r^\top X_{l_2} \geq 0 }\right] \succ \frac{\delta}{100 n^2}
\end{align*}
As a result, we can immediately get that $\H^{\cts} \succ \frac{\delta }{100 n^2}$.
\end{proof}

\subsection{Shifted NTK}\label{sec:separation:shifted}

In this section, we consider the shifted GNTK, which is defined by a shifted ReLU function.
We proved that given the shifted GNTK, the GNTK is still positive definite. Moreover, the eigenvalue of GNTK goes to $0$ exponentially when the absolute value of the bias $b$ goes larger.

We first state a claim that helps us decompose the eigenvalue of Hadamard product.
\begin{claim}[\cite{s11}]\label{clm:eigen_min}
Let 
\begin{itemize}
    \item $M_1,M_2\in \R^{n\times n}$ be two PSD matrices.
    \item $M_1 \circ M_2$ denote the Hadamard product of $M_1$ and $M_2$.
\end{itemize}
 Then,
\begin{align*}
    \lambda_{\min}(M_1 \circ M_2) \geq &~ (\min_{i\in [n]} {M_2}_{i,i})\cdot \lambda_{\min}(M_1), \\
    \lambda_{\max}(M_1 \circ M_2) \leq &~ (\max_{i\in [n]} {M_2}_{i,i})\cdot \lambda_{\max}(M_1).
\end{align*}
\end{claim}

Then, we state a result that bounds the eigenvalue of NTK. 
\begin{theorem}[\cite{syz21}]\label{thm:sep}
Let
\begin{itemize}
    \item $x_1,\dots,x_n$ be points in $\R^d$ with unit Euclidean norm and $w\sim{\cal N}(0,I_d)$.
    \item  Form the matrix $X \in \R^{n\times d}=[x_1~\dots~x_n]^\top$.
\end{itemize}
 Suppose there exists $\delta \in (0,\sqrt{2})$ such that 
\begin{align*}
    \min_{i \neq j \in [n]} \{\|x_i-x_j\|_2,\|x_i+x_j\|_2\}\geq \delta.
\end{align*}
Let $b\geq 0$. Recall the continuous Hessian matrix $H^{\cts}$ is defined by
\begin{align*}
    H^{\cts}_{i,j}:=\E_{w \sim \N(0,I)} \left[ x_i^\top x_j {\bf 1}_{ w^\top x_i \geq b, w^\top x_j \geq b } \right]~~~\forall (i,j)\in [n]\times [n].
\end{align*}
Let $\lambda:=\lambda_{\min}(H^{\cts})$. Then, we have
\begin{align}\label{eq:require}
    \exp(-b^2/2) \geq~ \lambda \geq~\exp(-b^2/2) \cdot \frac{\delta}{100n^2}.
\end{align}
\end{theorem}

We first state our definition of shifted graph neural network function, which can be viewed as a generalization of the standard graph neural network.
\begin{definition}[Shifted graph neural network function]
\label{def_generalize:f_nn}
We define a graph neural network 
$f_{\gnn, \mathrm{shift}}(W, a, b, G): \R^{d\times m}\times \R^m\times \R \times \R^{d \times N}\rightarrow\R$
as the following form
\begin{align*}
   f_{\gnn, \mathrm{shift}} (W, a, b, G) = \frac{1}{\sqrt{m}} \sum_{r=1}^m a_r \sum_{l =1}^N \sigma(w_r^\top \sum_{v\in {\cal N}(u_l)} h(G)_v-b). 
\end{align*}

Besides, We denote 
$f_{\gnn, \mathrm{shift}}(W, a, b, \G) = [f_{\gnn, \mathrm{shift}}(W, a, b,G_1),\cdots, f_{\gnn, \mathrm{shift}}(W,a,b, G_n)]^\top$
to represent the GNN output on the training set.
\end{definition}

Next, we define shifted graph neural tangent kernel for the shifted graph neural network function.
\begin{definition}[Shifted graph neural tangent kernel]\label{def_gen:ntk_phi}
We define the graph neural tangent kernel (GNTK) corresponding to the graph neural networks $f_{\gnn, \mathrm{shift}}$ as follows: 
\begin{align*}
  	\k_{\gntk, \mathrm{shift}}(G, H) = \E_W \left[\left\langle \frac{\partial f_{\gnn, \mathrm{shift}}(W,G)}{\partial W},\frac{\partial f_{\gnn, \mathrm{shift}}(W,H)}{\partial W} \right\rangle \right],  
\end{align*}
 where 
 \begin{itemize}
     \item  $G,H$ are any input graph,
     \item  and the expectation is taking over  $w_r\overset{i.i.d.}{\sim} \N(0,I),~r=1, \cdots, m $.
 \end{itemize}

We define $H^{\cts}_{ \mathrm{shift}}\in\R^{n\times n}$ as the kernel matrix between training data as $[H^{\cts}_ \mathrm{shift}]_{i,j} = \k_{\gntk, \mathrm{shift}}(G_i, G_j) \in \R$. 
\end{definition}

Then, we prove our eigenvalue bounds for shifted graph neural tangent kernel. 
\begin{theorem}
Let us define
\begin{align*}
     h_{G, u} = \sum_{ v \in \mathcal{N}(u) } h(G)_v, ~x_{i, p}=h_{G_i,u_p}, ~i\in[n], p\in[N] .
\end{align*}
Let $w_r$ be defined as in Definition \ref{def_generalize:f_nn}. Assuming that  $w_r\overset{i.i.d.}{\sim} \N(0,I),~r=1, \cdots, m$. 
Assuming that our datasets satisfy the $\delta$-separation: For any $i,j \in[n], ~p,q\in[N]$
\begin{align*}
\min\{\| \ov{x}_{i, p} - \ov{x}_{j,q} \|_2, \| \ov{x}_{i, p} + \ov{x}_{j,q} \|_2    \}  \geq \delta .
\end{align*}
where $\ov{x}_{i,p}= x_{i,p}/\|x_{i,p}\|_2$.

We define $H^{\cts}_{\mathrm{shift}}$ as in Definition \ref{def_gen:ntk_phi}. Then we can claim that, for $\lambda:=\lambda_{\min} (H^{\cts}_{\mathrm{shift}}) $
\begin{align*}
 \exp(-b^2/2) \geq \lambda \geq \exp(-b^2/2)\cdot \frac{\delta}{100 n^2}
\end{align*}
\end{theorem}

\begin{proof}
Because
\begin{align*}
     [\H^{\cts}]_{i,j}    =\E_W \left[\frac{1}{m}  \sum_{r=1}^m \sum_{l_1 = 1}^N \sum_{l_2 = 1}^N x_{i,l_1}^\top x_{j,l_2} {\bf 1}_{ w_r^\top x_{i,l_1} \geq 0, w_r^\top x_{j,l_2} \geq 0 } \right].
\end{align*}
we will claim that
\begin{align*}
    \H^{\cts} =&~ \E_W \left[ \frac{1}{m}  \sum_{r=1}^m \sum_{l_1 = 1}^N \sum_{l_2 = 1}^N X_{l_1} X_{l_2}^\top \odot {\bf 1}_{ w_r^\top X_{l_1} \geq b} {\bf 1}_{w_r^\top X_{l_2} \geq b } \right].\\
    =&~\frac{1}{m}  \sum_{r=1}^m \E_W \left[  (\sum_{l_1 =  1}^N X_{l_1}\odot {\bf 1}_{ w_r^\top X_{l_1} \geq b} ) ( \sum_{l_2 = 1}^N X_{l_2}^\top \odot {\bf 1}_{w_r^\top X_{l_2} \geq b })    \right].
\end{align*}
where $X_{i}=[x_{1,i}, \cdots, x_{n,i}]^\top \in\R^{d \times n}$.
By Theorem \ref{thm:sep}, we can get that
\begin{align*}
   \exp(-b^2/2) \succ \E_W\left[ X_{l_1} X_{l_2}^\top \odot {\bf 1}_{ w_r^\top X_{l_1} \geq b} {\bf 1}_{w_r^\top X_{l_2} \geq b }\right] \succ \exp(-b^2/2)\cdot \frac{\delta}{100 n^2}
\end{align*}
As a result, we can immediately get that $\exp(-b^2/2) \succ \H^{\cts} \succ \exp(-b^2/2) \frac{\delta}{100 n^2}$.
\end{proof}

\section{Node-Level GNTK Formulation}
\label{sec:node_gntk_formulation}

In this section we derive the  GNTK formulation for the node level regression task.
We first review the definitions. Let $f^{(l)}_{(r)}$ define the output of the Graph Neural Network in the $l$-th level and $r$-th layer. The definition of {\aggregate} layer is as follows: 
\begin{align*}
    f^{(1)}_{(1)}(u) =   \sum_{v \in \neighbor(u) } h(G)_v,
\end{align*}
and
\begin{align*}
    f^{(l)}_{(1)}(u) =   \sum_{v \in \neighbor(u) }f^{(l-1)}_{(R)}(v), ~l\in \{2,3,\cdots, L\}.
\end{align*}

The definition of {\combine} layer is formulated as follows:
\begin{align*}
    f^{(l)}_{(r)}(u) = \sqrt{\frac{c_\sigma }{ m}}\sigma(W^{(l)}_{(r)} f^{(l)}_{(r-1)}(u)), ~l\in[L], r\in[R]
\end{align*}
where $c_\sigma$ is a scaling factor, $W^{(l)}_{(r)}$ is the weight of the neural network in $l$-th level and $r$-th layer, $m$ is the width of the layer. 

We define the Graph Neural Tangent Kernel as $\k_{\gntk,\node}(u, u')=\k_{(R)}^{(L)}(u,u')$. Let $\bm{\Sigma}^{(l)}_{(r)}(u,u')$ be defined as the covariance matrix of the Gaussian process of the  pre-activations $W^{(l)}_{(r)} f^{(l)}_{(r-1)}(u)$. 
We have that initially:
\begin{align*}
    \bm{\Sigma}^{(1)}_{(1)}(u,u') =  h(G)_u^\top h(G)_{u'}.
\end{align*}

For $l\in \{2,3,\cdots, L\}$,
\begin{align*}
    \bm{\Sigma}^{(l)}_{(1)}(u,u') =  \sum_{v \in \neighbor(u) }\sum_{v' \in \neighbor(u') } \bm{\Sigma}^{(l-1)}_{(R)}(v,v').
\end{align*}

Then, we start to deduce the formulation for the {\combine} layer. Note that,
\begin{align*}
&~\E[(W^{(l)}_{(r+1)}f^{(l)}_{(r)}(u))_i \cdot (W^{(l)}_{(r+1)}f^{(l)}_{(r)}(u'))_i]\\
=&~ f^{(l)}_{(r)}(u) \circ f^{(l)}_{(r)}(u')\\
=&~ \frac{c_\sigma}{m} \sigma(W^{(l)}_{(r)}f^{(l)}_{(r-1)}(u)) \circ \sigma(W^{(l)}_{(r)}f^{(l)}_{(r-1)}(u'))
\end{align*}

Then, we have that
\begin{align*}
\lim_{m\to \infty} \E[(W^{(l)}_{(r+1)}f^{(l)}_{(r)}(u))_i\cdot  (W^{(l)}_{(r+1)}f^{(l)}_{(r)}(u'))_i] = \bm{\Sigma}^{(l)}_{(r+1)}(u, u'),
\end{align*}
where $W^{(l)}_{(r+1)}\in\R^{m\times m}$.

We conclude the formula for the {\combine} layer. 
\begin{align*}
\bm{\Lambda}^{(l)}_{(r)}(u,u') =& \begin{pmatrix}
\bm{\Sigma}_{(r-1)}^{(l)}(u,u)  & \bm{\Sigma}_{(r-1)}^{(l)}(u,u') \\
\bm{\Sigma}_{(r-1)}^{(l)}(u',u) & \bm{\Sigma}_{(r-1)}^{(l)}(u',u') 
\end{pmatrix} \in \mathbb{R}^{2 \times 2},\\
\bm{\Sigma}^{(l)}_{(r)}(u,u') = &c_{\sigma}\mathbb{E}_{(a,b) \sim \mathcal{N}\left(\bm{0},\bm{\Lambda}^{(l)}_{(r)}\left(u,u'\right)\right)}[\sigma(a)\sigma(b)], 
\\
\bm{\dot{\Sigma}}^{(l)}_{(r)}\left(u,u'\right) = &c_\sigma\mathbb{E}_{(a,b) \sim \mathcal{N}\left(\bm{0},\bm{\Lambda}^{(l)}_{(r)}\left(u,u'\right)\right) }[\dot{\sigma}(a)\dot{\sigma}(b)]. 
\end{align*}



We write the partial derivative with respect to a particular weight $W^{(l)}_{(r)}$:
\begin{align*}
    \frac{\partial f^{(L)}_{(R)}}{\partial  W^{(l)}_{(r)}} = \frac{\partial f^{(L)}_{(R)}}{\partial f^{(l)}_{(r)}} \sqrt{\frac{c_\sigma}{m}} \diag(\dot{\sigma}(W^{(l)}_{(r)} f^{(l)}_{(r-1)}))f^{(l)}_{(r-1)}
\end{align*}
where
\begin{align*}
\frac{\partial f^{(L)}_{(R)}}{\partial f^{(l)}_{(r-1)}} = \frac{\partial f^{(L)}_{(R)}}{\partial f^{(l)}_{(r)}} \sqrt{\frac{c_\sigma}{m}} \diag(\dot{\sigma}(W^{(l)}_{(r)} f^{(l)}_{(r-1)}))W^{(l)}_{(r)}
\end{align*}

Then, we compute the inner product:
\begin{align*}
&~\frac{\partial f^{(L)}_{(R)}(u)}{\partial  W^{(l)}_{(r)}} \circ \frac{\partial f^{(L)}_{(R)}(u')}{\partial  W^{(l)}_{(r)}}\\
=&~  \langle \frac{\partial f^{(L)}_{(R)}}{\partial f^{(l)}_{(r)}} \sqrt{\frac{c_\sigma}{m}} \diag(\dot{\sigma}(W^{(l)}_{(r)} f^{(l)}_{(r-1)}))f^{(l)}_{(r-1)}(u), \frac{\partial f^{(L)}_{(R)}}{\partial f^{(l)}_{(r)}} \sqrt{\frac{c_\sigma}{m}} \diag(\dot{\sigma}(W^{(l)}_{(r)} f^{(l)}_{(r-1)}))f^{(l)}_{(r-1)}(u') \rangle\\
=&~\langle \frac{\partial f^{(L)}_{(R)}}{\partial f^{(l)}_{(r)}} \sqrt{\frac{c_\sigma}{m}} \diag(\dot{\sigma}(W^{(l)}_{(r)} f^{(l)}_{(r-1)}))(u), \frac{\partial f^{(L)}_{(R)}}{\partial f^{(l)}_{(r)}} \sqrt{\frac{c_\sigma}{m}} \diag(\dot{\sigma}(W^{(l)}_{(r)} f^{(l)}_{(r-1)}))(u') \rangle \cdot \langle f^{(l)}_{(r-1)}(u), f^{(l)}_{(r-1)}(u') \rangle\\
=&~\langle \frac{\partial f^{(L)}_{(R)}}{\partial f^{(l)}_{(r)}} \sqrt{\frac{c_\sigma}{m}} \diag(\dot{\sigma}(W^{(l)}_{(r)} f^{(l)}_{(r-1)}))(u), \frac{\partial f^{(L)}_{(R)}}{\partial f^{(l)}_{(r)}} \sqrt{\frac{c_\sigma}{m}} \diag(\dot{\sigma}(W^{(l)}_{(r)} f^{(l)}_{(r-1)}))(u') \rangle \cdot {\bf \Sigma}^{(l)}_{(r)} (u, u')
\end{align*}
where
\begin{align*}
    &~ \langle \frac{\partial f^{(L)}_{(R)}}{\partial f^{(l)}_{(r)}} \sqrt{\frac{c_\sigma}{m}} \diag(\dot{\sigma}(W^{(l)}_{(r)} f^{(l)}_{(r-1)})(u)), \frac{\partial f^{(L)}_{(R)}}{\partial f^{(l)}_{(r)}} \sqrt{\frac{c_\sigma}{m}} \diag(\dot{\sigma}(W^{(l)}_{(r)} f^{(l)}_{(r-1)}(u'))) \rangle \\
    = &~\frac{c_\sigma}{m}\tr[ \diag(\dot{\sigma}(W^{(l)}_{(r)} f^{(l)}_{(r-1)}))(u)  \diag(\dot{\sigma}(W^{(l)}_{(r)} f^{(l)}_{(r-1)}))(u')] \langle \frac{\partial f^{(L)}_{(R)}(u)}{\partial f^{(l)}_{(r)}}, \frac{\partial f^{(L)}_{(R)}(u')}{\partial f^{(l)}_{(r)}} \rangle \\
    = &~\frac{c_\sigma}{m} \dot{\bf \Sigma}^{(l)}_{(r)}(u,u') \langle \frac{\partial f^{(L)}_{(R)}(u)}{\partial f^{(l)}_{(r)}}, \frac{\partial f^{(L)}_{(R)}(u')}{\partial f^{(l)}_{(r)}} \rangle
\end{align*}

Finally, we conclude the formula for  graph neural tangent kernel. For each {\aggregate} layer:
\begin{align*}
    \k^{(1)}_{(1)}(u,u') =~ h(G)_u^\top h(G)_{u'},
\end{align*}
and 
\begin{align*}
 \k^{(l)}_{(1)}(u,u') =~ \sum_{v \in \neighbor(u) }\sum_{v' \in \neighbor(u')} \k^{(l-1)}_{(R)}(v,v').   
\end{align*}
For the {\combine} layer:
\begin{align*}
    \k_{(r)}^{(l)}(u,u') = \k_{(r-1)}^{(l)}(u,u') \dot{\bm{\Sigma}}^{(l)}_{(r)}\left(u,u'\right)  + \bm{\Sigma}^{(l)}_{(r)}\left(u,u'\right).
\end{align*}
The Graph neural tangent kernel:
\begin{align*}
  \k_{\gntk,\node}(u, u')  =\k_{(R)}^{(L)}(u,u').
\end{align*}

\end{document}